\documentclass{article} 
\usepackage[table]{xcolor}
\usepackage{eso-pic} 
\usepackage{fancyhdr}
\usepackage{hyperref}
\usepackage{url}
\usepackage{microtype}
\usepackage{graphicx}
\usepackage{subcaption}
\usepackage{booktabs} 
\usepackage{amsmath}
\usepackage{amsthm}
\usepackage{amssymb}
\usepackage{algorithmic}
\usepackage{eso-pic}
\usepackage{forloop}
\usepackage{natbib}
\usepackage{comment}
\usepackage{mathtools}
\usepackage{thmtools, thm-restate}
\usepackage{multirow}
\usepackage{bigdelim}
\usepackage{makecell}
\usepackage[shortlabels]{enumitem}
\usepackage{cases}
\usepackage[accepted]{icml2020}
\newif\ifcomment
\commenttrue

\newcommand{\pseudoy}[1]{f_{\stdest}(\tilde{x}_{#1})}

\newcommand{\distribx}{P_\mathsf{x}}
\newcommand{\distriby}[1]{P_\mathsf{y}(\cdot \mid #1)}
\newcommand{\distribxy}{P_\mathsf{xy}}
\newcommand{\xstd}{X_\text{std}}
\newcommand{\ystd}{y_\text{std}}

\newcommand{\stderr}{L_{\text{std}}}
\newcommand{\roberr}{L_{\text{rob}}}

\newcommand{\xext}{X_\text{ext}}
\newcommand{\yext}{y_\text{ext}}
\newcommand{\minmat}{M}
\newcommand{\stdest}{\hat{\theta}_\text{std}}
\newcommand{\augest}{\hat{\theta}_\text{aug}}
\newcommand{\ext}{\text{ext}}
\newcommand{\rest}{\text{rest}}

\newcommand{\cifar}{\textsc{Cifar-10}}

\newcommand{\thetatrue}{\theta^{\star}}
\newcommand{\sigmapop}{\Sigma}
\newcommand{\sigmastd}{\Sigma_\text{std}}
\newcommand{\sigmaaug}{\Sigma_\text{aug}}
\newcommand{\sigmaext}{\Sigma_\text{ext}}

\newcommand{\pistd}{\Pi_{\text{std}}^\perp}
\newcommand{\piaug}{\Pi_{\text{aug}}^\perp}
\newcommand{\piparastd}{\Pi_{\text{std}}}
\newcommand{\piparaaug}{\Pi_{\text{aug}}}

\newcommand{\Null}{\text{Null}}
\newcommand{\Col}{\text{Col}}
\newcommand{\singleaug}{x_{\text{ext}}}
\newcommand{\deltastd}{\Delta_\text{std}}
\newcommand{\deltaaug}{\Delta_\text{aug}}

\newcommand{\thetarst}{\hat{\theta}_\text{rst}}

\newcommand{\thetainterp}{\theta_{\textup{int-std}}}
\DeclareMathOperator*{\argmin}{arg\,min}
\DeclareMathOperator*{\argmax}{arg\,max}

\newcommand{\defeq}{:=}
\newcommand{\xadv}{x_\text{adv}}
\newcommand{\linspan}{\mathop{\rm span}}
\newcommand{\opt}{^\star}
\newcommand{\norm}[1]{\left\|{#1}\right\|}

\newcommand{\supp}{\text{supp}}

\newcommand{\ellrob}{\ell_\text{rob}}
\newcommand{\Lsl}{\hat{L}_{\text{std-lab}}}
\newcommand{\Lsu}{\hat{L}_{\text{std-unlab}}}
\newcommand{\Lrl}{\hat{L}_{\text{rob-lab}}}
\newcommand{\Lru}{\hat{L}_{\text{rob-unlab}}}
\newcommand{\pLsu}{L_{\text{std-unlab}}}
\newcommand{\pLru}{L_{\text{rob-unlab}}}

\newcommand{\tline}{\sT_\text{line}}
\newcommand{\xline}{\sX_\text{line}}
\newcommand\uarsim{\stackrel{\mathclap{u.a.r}}{\sim}}
\newcommand{\pilocalglobal}{\Pi_{\text{lg}}}

\newcommand{\xone}{\mathbf{x_1}}
\newcommand{\xtwo}{\mathbf{x_2}}
\newcommand{\xthree}{\mathbf{x_3}}

\newlength{\widebarargwidth}
\newlength{\widebarargheight}
\newlength{\widebarargdepth}

\newcommand\sT{\ensuremath{\mathcal{T}}}

\newcommand\sX{\ensuremath{\mathcal{X}}}
\newcommand\sY{\ensuremath{\mathcal{Y}}}

\newcommand\bx{\ensuremath{\mathbf{x}}}

\DeclareMathOperator*{\diag}{diag} 

\newcommand\R{\ensuremath{\mathbb{R}}} 
\newcommand\eqdef{\ensuremath{\stackrel{\rm def}{=}}} 
\newcommand\refeqn[1]{(\ref{eqn:#1})}

\ifthenelse{\isundefined{\definition}}{\newtheorem{definition}{Definition}}{}
\ifthenelse{\isundefined{\assumption}}{\newtheorem{assumption}{Assumption}}{}
\ifthenelse{\isundefined{\hypothesis}}{\newtheorem{hypothesis}{Hypothesis}}{}
\ifthenelse{\isundefined{\proposition}}{\newtheorem{proposition}{Proposition}}{}
\ifthenelse{\isundefined{\theorem}}{\newtheorem{theorem}{Theorem}}{}
\ifthenelse{\isundefined{\lemma}}{\newtheorem{lemma}{Lemma}}{}
\ifthenelse{\isundefined{\corollary}}{\newtheorem{corollary}{Corollary}}{}
\ifthenelse{\isundefined{\alg}}{\newtheorem{alg}{Algorithm}}{}
\ifthenelse{\isundefined{\example}}{\newtheorem{example}{Example}}{}
\newcommand{\E}{\ensuremath{\mathbb{E}}} 

\icmltitlerunning{Understanding and Mitigating the Tradeoff Between Robustness and Accuracy}

\begin{document}

\twocolumn[
  \icmltitle{Understanding and Mitigating the Tradeoff Between Robustness and Accuracy}

\icmlsetsymbol{equal}{\textbf{*}}

\begin{icmlauthorlist}
\icmlauthor{Aditi Raghunathan}{equal,stan}
\icmlauthor{Sang Michael Xie}{equal,stan}
\icmlauthor{Fanny Yang}{eth}
\icmlauthor{John C. Duchi}{stan}
\icmlauthor{Percy Liang}{stan}
\end{icmlauthorlist}

\icmlaffiliation{stan}{Stanford University}
\icmlaffiliation{eth}{ETH Zurich}
\icmlcorrespondingauthor{Aditi Raghunathan}{aditir@stanford.edu}
\icmlcorrespondingauthor{Sang Michael Xie}{xie@cs.stanford.edu}

\icmlkeywords{Machine Learning, ICML}

\vskip 0.3in
]

\printAffiliationsAndNotice{\icmlEqualContribution} 

\begin{abstract}
Adversarial training augments the training set with perturbations to improve the robust
error (over worst-case perturbations), but it often leads to an increase in the standard error (on unperturbed test inputs). Previous explanations for this tradeoff rely on the assumption that no predictor in the hypothesis class has low standard and robust error. In this work, we precisely characterize the effect of augmentation on the standard error in linear regression when the optimal linear predictor has zero standard and robust error. In particular, we show that the standard error could increase even when the augmented perturbations have noiseless observations from the optimal linear predictor. We then prove that the recently proposed robust self-training (RST) estimator improves robust error without sacrificing standard error for noiseless linear regression. Empirically, for neural networks, we find that RST with different adversarial training methods improves both standard and robust error for random and adversarial rotations and adversarial $\ell_\infty$ perturbations in CIFAR-10.
\end{abstract}

\section{Introduction}
\label{sec:intro}

Adversarial training methods \citep{goodfellow2015explaining,madry2017towards}
attempt to improve the robustness of neural networks against adversarial
examples \citep{szegedy2014intriguing}
by augmenting the training set (on-the-fly) with perturbed examples that
preserve the label but that fool the current model.
While such methods decrease the \emph{robust error}, the error on worst-case
perturbed inputs, they have been observed to cause an undesirable increase in
the \emph{standard error}, the error on unperturbed inputs~\citep{madry2018towards, zhang2019theoretically, tsipras2019robustness}.

\begin{figure}
  \begin{minipage}[c]{0.45\linewidth}
    \includegraphics[width=\linewidth]{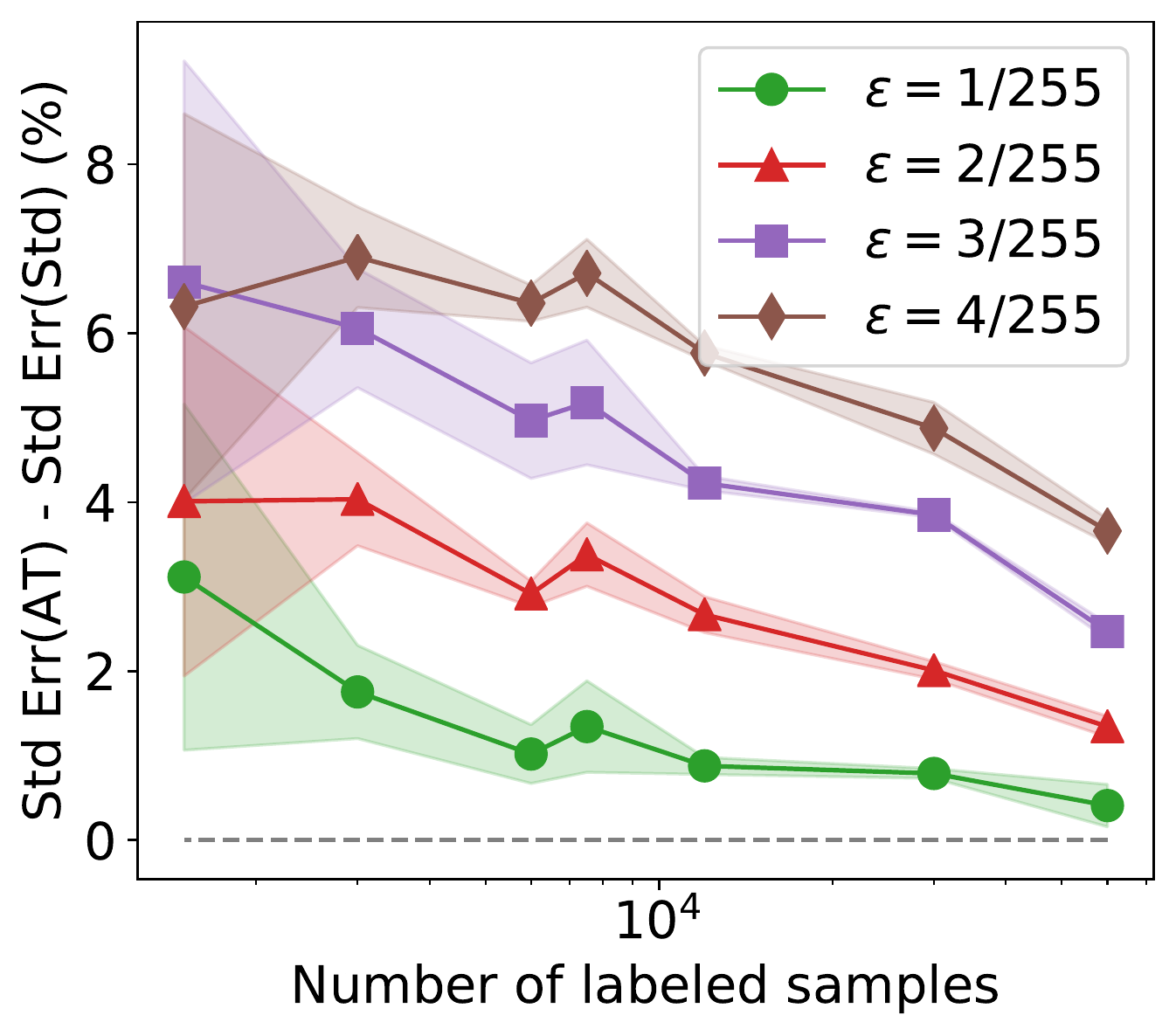}
  \end{minipage}\hfill
  \begin{minipage}[c]{0.54\linewidth}
    \caption{
      Gap between the standard error of adversarial trainning~\citep{madry2018towards} with $\ell_\infty$ perturbations, and standard training. The gap decreases with increase in training set size, suggesting that the tradeoff between standard and robust error should disappear with infinite data. 
    } \label{fig:sample-size}
  \end{minipage}
  \vspace{-10pt}
\end{figure}

Previous works attempt to explain the tradeoff between standard error
and robust error in two settings: when no accurate classifier is
\emph{consistent} with the perturbed
data~\citep{tsipras2019robustness, zhang2019theoretically,
  fawzi2018analysis}, and when the hypothesis class is not expressive
enough to contain the true classifier~\citep{nakkiran2019adversarial}.
In both cases, the tradeoff persists even with infinite data.
However, adversarial perturbations in practice are typically defined to be
imperceptible to humans (e.g. small $\ell_\infty$ perturbations in
vision). Hence by definition, there exists a classifier (the human)
that is both robust and accurate with no tradeoff in
the infinite data limit. Furthermore, since deep neural networks are
expressive enough to fit not only adversarial but also randomly
labeled data perfectly ~\citep{zhang2017understanding}, the explanation
of a restricted hypothesis class does not perfectly capture empirical
observations either. Empirically on~\cifar, we find that the gap between the standard error of adversarial training
and standard training decreases as we increase the labeled data size, thereby also suggesting the tradeoff could disappear with infinite data (See Figure~\ref{fig:sample-size}). 

\begin{figure*}[t]
  \centering
  \begin{subfigure}{.32\linewidth}
    \centering
    \includegraphics[scale=0.34]{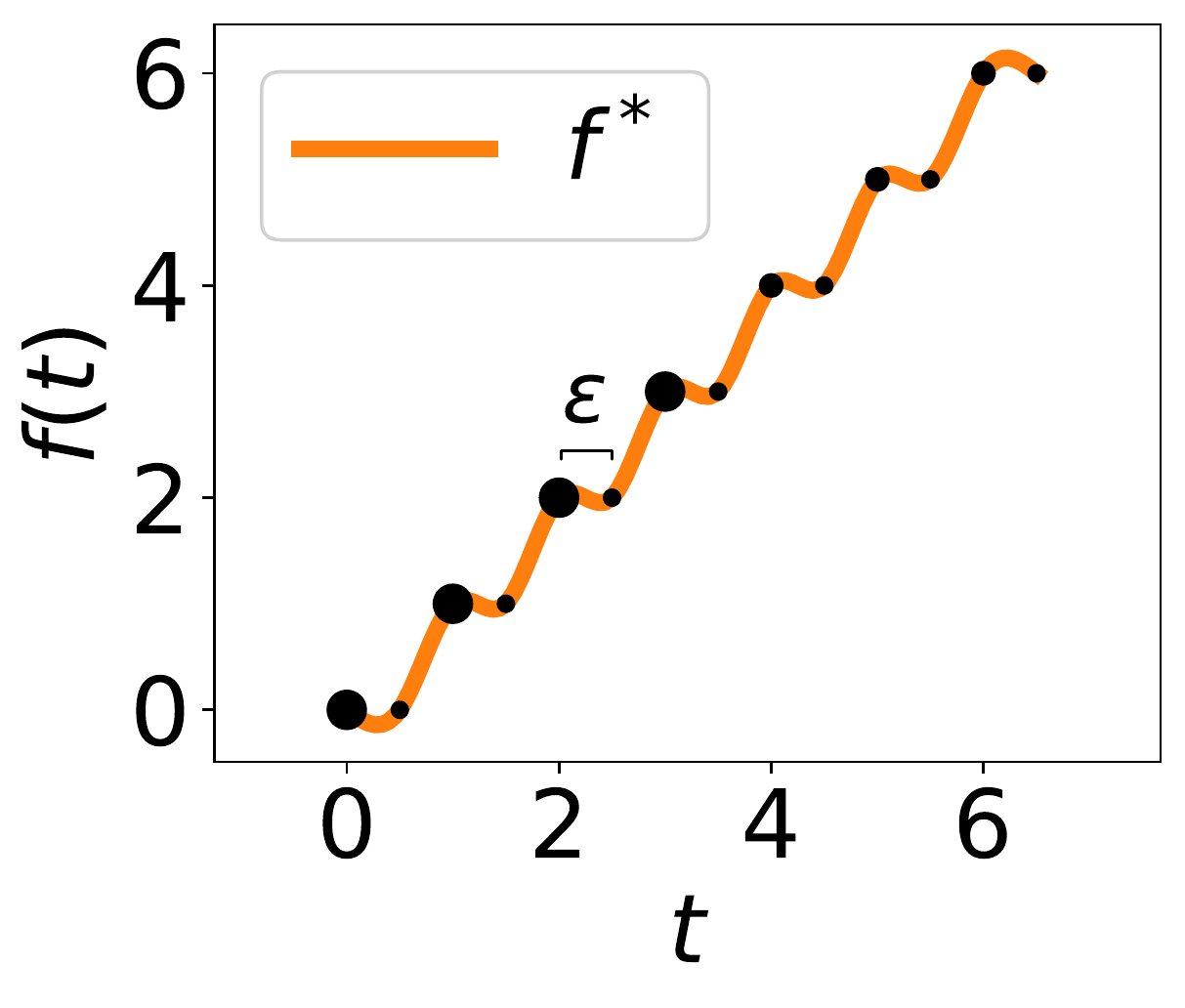}
  \end{subfigure}
  \begin{subfigure}{.32\linewidth}
    \centering
    \includegraphics[scale=0.4]{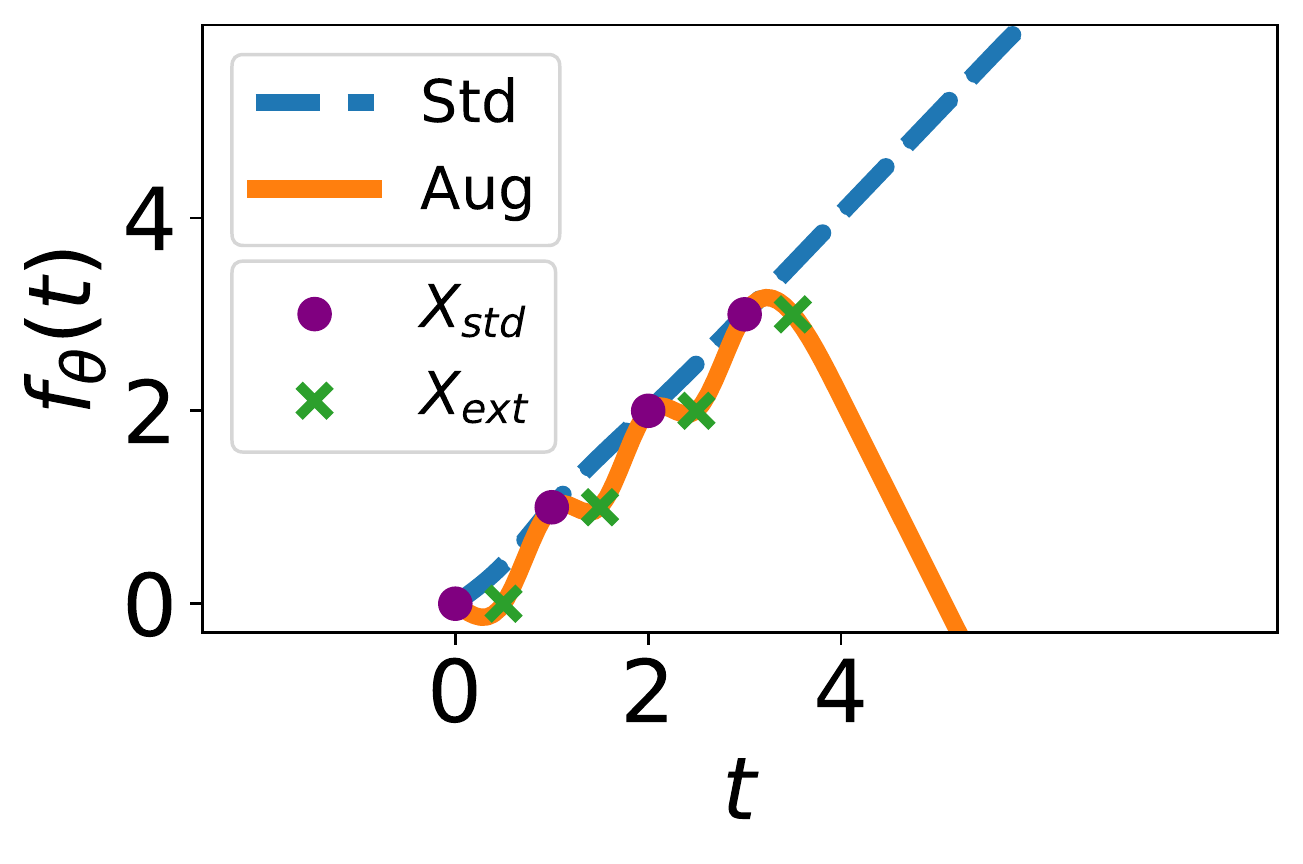}
  \end{subfigure}
  \begin{subfigure}{.32\linewidth}
    \centering
    \includegraphics[scale=0.4]{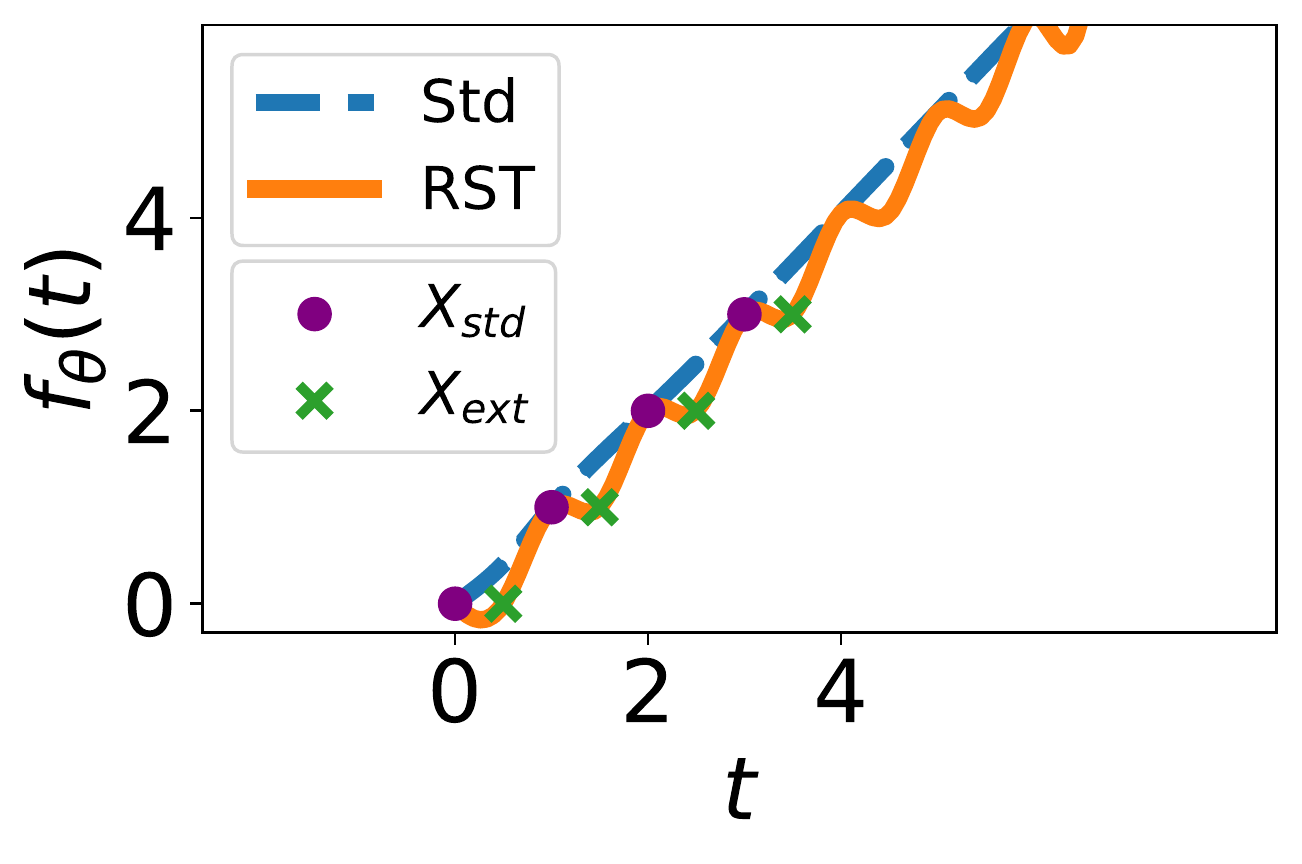}
  \end{subfigure}
  \caption{
    We consider function interpolation via cubic splines.
    \textbf{(Left)} The underlying distribution $\distribx$ denoted by sizes of the circles. The true function is a staircase. 
    \textbf{(Middle)} With a small number of standard training samples (purple circles), an augmented estimator that fits local perturbations (green crosses) has a large error. In constrast, the standard estimator that does not fit perturbations is a simple straight line and has small error. 
    \textbf{(Right)} Robust self-training (RST) regularizes the predictions of an augmented estimator towards the predictions of the standard estimator thereby obtaining both small error on test points and their perturbations. }
  \label{fig:spline}
\end{figure*}

In this work, we provide a different explanation for the tradeoff
between standard and robust error that takes \emph{generalization}
from finite data into account. We first consider a linear model where the true linear function
has zero standard and robust error. Adversarial training augments the
original training set with \emph{extra} data, consisting of samples
$(\singleaug, y)$ where the perturbations $\singleaug$ are \emph{consistent},
meaning that the conditional distribution stays constant
$\distriby{\singleaug} = \distriby{x}$.
We show that even in this simple setting, the \emph{augmented estimator}, i.e. the minimum norm
interpolant of the augmented data (standard + extra data),
could have a larger standard error than that of the \emph{standard estimator},
which is the minimum norm interpolant of the standard data alone.
We found this surprising given that adding consistent perturbations
enforces the predictor to satisfy invariances that the true model
exhibits. One might think adding this information would only restrict the hypothesis class and thus enable better generalization, not worse.

We show that this tradeoff stems from overparameterization. If the \emph{restricted} hypothesis class (by
enforcing invariances) is still overparameterized, the inductive bias of the estimation procedure (e.g., the norm being minimized) plays a key role in determining the generalization of a model.

Figure~\ref{fig:spline} shows an illustrative example of this phenomenon
with cubic smoothing splines.
The predictor obtained via standard training (dashed blue) is a line that captures the global
structure and obtains low error. Training on augmented data with
locally consistent perturbations of the training data (crosses)
restricts the hypothesis class by encouraging the predictor to fit the
local structure of the high density points. Within this set,
the cubic splines predictor (solid orange) minimizes the second
derivative on the augmented data, compromising the global structure
and performing badly on the tails (Figure~\ref{fig:spline}(b)).
More generally, as we characterize in Section~\ref{sec:min-norm}, the
tradeoff stems from the inductive bias of the minimum norm
interpolant, which minimizes a fixed norm independent of the data, while the standard error depends on the geometry of the
covariates.

Recent works~\citep{carmon2019unlabeled, najafi2019robustness, uesato2019are} introduced robust self-training (RST), a robust variant of self-training that overcomes the sample complexity barrier of learning a model with low robust error by leveraging extra unlabeled data.  In this paper, our theoretical understanding of the tradeoff between standard and robust error in linear regression motivates RST as a method to improve robust error without sacrificing standard error.  In Section~\ref{sec:linear-rst}, we prove that RST \emph{eliminates} the tradeoff
for linear regression---RST does not increase standard error compared to the standard estimator
while simultaneously achieving the best possible robust error, matching the standard error (see
Figure~\ref{fig:spline}(c) for the effect of RST on the spline problem). Intuitively, RST regularizes the predictions of the robust estimator towards that of the standard estimator on the unlabeled data thereby eliminating the tradeoff. 

As previous works only focus on the empirical evaluation of the gains in robustness via RST, we systematically
evaluate the effect of RST on \emph{both} the standard and robust error on~\cifar~when using unlabeled data from Tiny Images as sourced in~\citet{carmon2019unlabeled}.
We expand upon empirical results in two ways.  First, we study the effect of the labeled training set sizes and and find that the RST improves both robust and standard error over vanilla adversarial training across \emph{all} sample sizes. RST offers maximum gains at smaller sample sizes where vanilla adversarial training increases the standard error the most. Second, we consider an additional family of perturbations over random and adversarial rotation/translations and find that RST offers gains in both robust and standard error.

\section{Setup}
\label{sec:setting}
We consider the problem of learning a mapping from an input $x\in \sX \subseteq \R^d$ to a target $y \in \sY$.
For our theoretical analysis, we focus on regression where $\sY=\R$ while our empirical studies consider general $\sY$.
Let $\distribxy$ be the underlying distribution, $\distribx$ the marginal on the inputs and $\distriby{x}$ the conditional distribution of the targets given inputs.
Given $n$ training pairs $(x_i, y_i) \sim \distribxy$,
we use $\xstd$ to denote the measurement matrix $[x_1, x_2, \hdots x_n]^\top \in \R^{n\times d}$ and $\ystd$ to denote the target vector $[y_1, y_2, \hdots y_n]^\top \in \R^n$.
Our goal is to learn a predictor $f_\theta: \sX \to \sY$ that (i) has low standard error on inputs $x$ and (ii) low robust error with respect to a set of perturbations $T(x)$. Formally, the error metrics for a predictor $f_\theta$ and a loss function $\ell$ are the \emph{standard error}
\begin{align}
  \label{eqn:test-error}
  \stderr(\theta) &= \E_{\distribxy}[\ell(f_{\theta}(x), y)]
\end{align}
and the \emph{robust error}
\begin{align}
  \label{eqn:rob-test-error}
  \roberr(\theta) &= \E_{\distribxy}[\max_{\singleaug \in T(x)} \ell(f_\theta(\singleaug), y)], 
\end{align}
for \emph{consistent perturbations} $T(x)$ that satisfy 
\begin{align}
  \label{eqn:target-preserving}
  \distriby{\singleaug} = \distriby{x}, ~~~\forall \singleaug \in T(x).
\end{align}
Such transformations may consist of small rotations, horizontal flips, brightness or contrast changes~\citep{krizhevsky2012imagenet, yaeger1996effective}, or small $\ell_p$ perturbations in vision~\citep{szegedy2014intriguing, goodfellow2015explaining} or word synonym replacements in NLP~\citep{jia2017adversarial, alzantot2018adversarial}. 

\paragraph{Noiseless linear regression.}
\label{sec:linear-model}
In section~\ref{sec:min-norm}, we analyze noiseless linear regression on inputs $x$ with targets $y = x^\top \thetatrue$ with true parameter $\thetatrue \in \R^k$.\footnote{Our analysis extends naturally to arbitrary feature maps $\phi(x)$.} For linear regression, $\ell$ is the squared loss which leads to the standard error (Equation~\ref{eqn:test-error}) taking the form 
\begin{equation}
    \label{eqn:linear-std-error}
  \stderr(\theta) = \E_{\distribx}[(x^\top \theta - x^\top \theta^\star)^2] = (\theta-\theta^\star)^\top \sigmapop (\theta-\theta^\star),
\end{equation}
where $\sigmapop = \E_{\distribx}[xx^\top]$ is the population covariance.

\paragraph{Minimum norm estimators.} In this work, we focus on interpolating estimators in highly overparameterized models, motivated by modern machine learning models that achieve near zero training loss (on both standard and extra data). Interpolating estimators for linear regression have been studied in many recent works such as~\citep{ma2018power, belkin2018understand, hastie2019surprises, liang2018just, bartlett2019benign}. 
We present our results for interpolating estimators with minimum Euclidean norm, but our analysis directly applies to more
general Mahalanobis norms via suitable reparameterization (see
Appendix~\ref{sec:app-matrix-norms}).

We consider robust training approaches that augment the
\emph{standard} training data $\xstd,\ystd \in \R^{n \times d} \times\R$ with some \emph{extra} training data $\xext,\yext \in \R^{m \times
  d} \times \R$ where the rows of $\xext$ consist of vectors in the
set $\{\singleaug: \singleaug \in T(x), x\in \xstd\}$.\footnote{In practice, $\xext$ is typically generated via iterative
  optimization such as in adversarial training~\citep{madry2018towards}, or by random sampling as in data
  augmentation~\citep{krizhevsky2012imagenet,yaeger1996effective}.}
We call the standard data together with the extra data as \emph{augmented} data. We compare the following min-norm estimators: (i) the \emph{standard estimator} $\stdest$
interpolating $[\xstd, \ystd]$ and (ii) the
\emph{augmented estimator} $\augest$ interpolating $X = [\xstd; \xext], Y =
     [\ystd; \yext]$:
\begin{align}
  \stdest &= \arg \min \limits_{\theta} \Big \{ \| \theta \|_2 : \xstd \theta = \ystd \Big \} \nonumber \\
  \augest &= \arg \min \limits_{\theta} \Big \{ \| \theta \|_2 : \xstd \theta = \ystd, \xext \theta = \yext \Big \}. \label{eqn:est} 
\end{align}

\paragraph{Notation.} For any vector $z \in \R^n$, we use $z_i$ to denote the $i^\text{th}$ coordinate of $z$.

\section{Analysis in the linear regression setting}
\label{sec:min-norm}
In this section, we compare the standard errors of the standard
estimator and the augmented estimator in noiseless linear
regression. We begin with a simple toy example that describes the
intuition behind our results (Section~\ref{sec:simple-3D}) and
provide a more complete characterization in
Section~\ref{sec:general}. This section focuses only on the standard
error of both estimators; we revisit the robust error together with the standard error in
Section~\ref{sec:rst}.
\subsection{Simple illustrative problem}
\label{sec:simple-3D}
We consider a simple example in 3D where $\thetatrue \in \R^3$
is the true parameter.
Let $e_1 = [1, 0, 0]; e_2 = [0, 1, 0]; e_3 = [0, 0, 1]$ denote the standard basis vectors in $\R^3$. Suppose
we have one point in the standard training data $\xstd = [0, 0, 1]$.
By definition~\eqref{eqn:est}, $\stdest$ satisfies $\xstd\stdest = \ystd$ and hence $(\stdest)_3 = \thetatrue_3$.
However, $\stdest$ is unconstrained on the subspace spanned by $e_1, e_2$ (the nullspace $\Null(\xstd)$).
The min-norm objective chooses the solution with $(\stdest)_1 = (\stdest)_2 = 0$.
Figure~\ref{fig:simpleexistence} visualizes the projection of various quantities on $\Null(\xstd)$. For simplicity of presentation, we omit the projection operator in the figure. The projection of $\stdest$ onto $\Null(\xstd)$ is the blue dot at the origin, and the parameter error $\thetatrue - \stdest$ is the projection of $\thetatrue$ onto $\Null(\xstd)$.

\paragraph{Effect of augmentation on parameter error.}
Suppose we augment with an extra data point $\xext = [1, 1, 0] = e_1 + e_2$ which lies in $\Null(\xstd)$ (black dashed line in Figure~\ref{fig:simpleexistence}).
The augmented estimator $\augest$ still fits the standard data $\xstd$ and thus $(\augest)_3 = \thetatrue_3 = (\stdest)_3$. Due to fitting the extra data $\xext$, $\augest$ (orange vector in Figure~\ref{fig:simpleexistence}) must also satisfy an additional constraint $\xext\augest = \xext\thetatrue$. 
The crucial observation is that additional constraints along one direction ($e_1 + e_2$ in this case) could actually increase parameter error along other directions.
For example, let's consider the direction $e_2$ in Figure~\ref{fig:simpleexistence}. Note that fitting $\xext$ makes $\augest$ have a large component along $e_2$. Now if $\thetatrue_2$ is small (precisely, $\thetatrue_2 < \thetatrue_1/3$), $\augest$ has a larger parameter error along $e_2$ than $\stdest$, which was simply zero (Figure~\ref{fig:simpleexistence} (a)). Conversely, if the true component $\thetatrue_2$ is large enough (precisely, $\thetatrue_2 > \thetatrue_1/3$), the parameter error of $\augest$ along $e_2$ is smaller than that of $\stdest$.

\paragraph{Effect of parameter error on standard error.}
The contribution of different components of the parameter error to the standard error is scaled by the population covariance $\sigmapop$ (see Equation~\ref{eqn:linear-std-error}). For simplicity, let $\sigmapop = \diag([\lambda_1, \lambda_2, \lambda_3])$.
In our example, the parameter error along $e_3$ is zero since both estimators interpolate the
standard training point $\xstd = e_1=3$.
Then, the ratio between
$\lambda_1$ and $\lambda_2$ determines which component of the
parameter error contributes more to the standard error.

\paragraph{When is $\stderr(\augest) > \stderr(\stdest)$?}
Putting the two effects together, we see that when $\thetatrue_2$ is small as in Fig~\ref{fig:simpleexistence}(a), $\augest$ has larger parameter error than $\stdest$ in the direction $e_2$. If $\lambda_2 \gg
\lambda_1$, error in $e_2$ is weighted much more heavily in the standard error and consequently $\augest$ would have a larger standard error.
Precisely, we have
\begin{align*}
\stderr(\augest) > \stderr(\stdest) \iff \lambda_2 (\thetatrue_1 - 3\thetatrue_2)  > \lambda_1 (3 \thetatrue_1 - \thetatrue_2). 
\end{align*}
We present a formal characterization of this tradeoff in general in the next section. 
\begin{figure}[tbp]
  \begin{subfigure}{0.49\linewidth}
      \centering
      \includegraphics[scale=0.60]{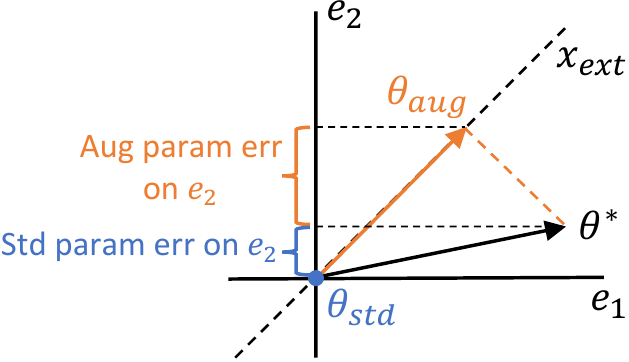}
      \caption{$\thetatrue_1\gg \thetatrue_2$}
  \end{subfigure}
  \begin{subfigure}{0.49\linewidth}
      \centering
      \includegraphics[scale=0.60]{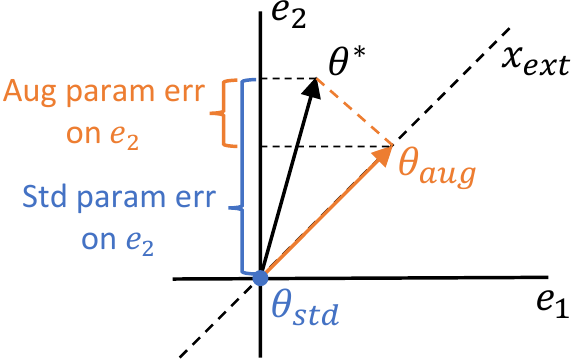}
      \caption{$\thetatrue_2\gg \thetatrue_1$}
  \end{subfigure}
  \begin{subfigure}{0.49\linewidth}
      \centering
    \includegraphics[scale=0.20]{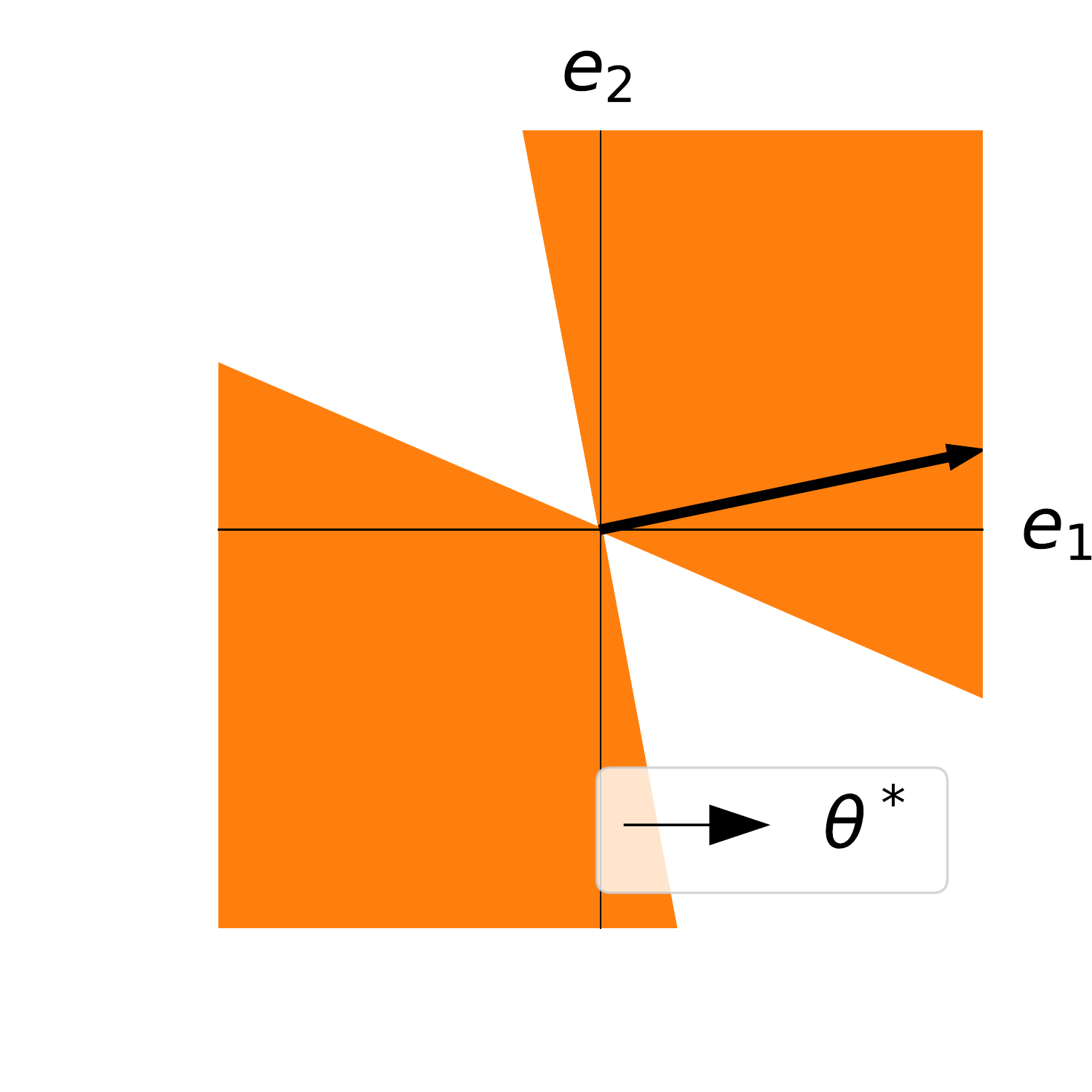}
    \caption{$\sigmapop=\diag([1, 4])$}
  \end{subfigure}
  \begin{subfigure}{0.49\linewidth}
      \centering
    \includegraphics[scale=0.20]{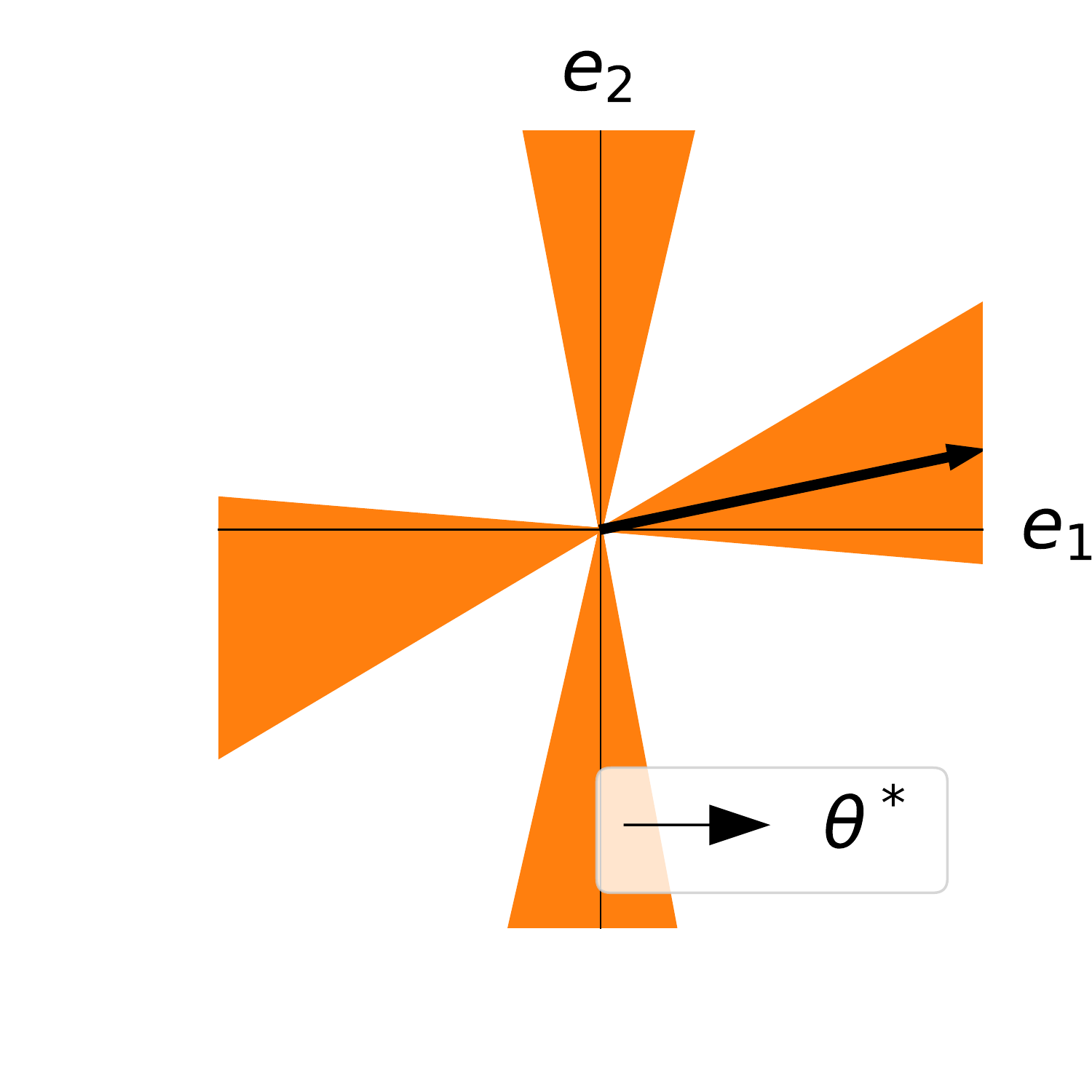}
    \caption{$\sigmapop=\diag([1, 25])$}
  \end{subfigure}
    \caption{
      Illustration of the 3-D example described in Sec.~\ref{sec:simple-3D}.
      \textbf{(a)-(b) Effect of augmentation on parameter error for different $\thetatrue$.} We show the projections of the standard estimator $\stdest$ (blue circle), augmented estimator $\augest$ (orange arrow), and true parameters $\thetatrue$ (black arrow) on $\Null(\xstd)$, spanned by $e_1$ and $e_2$. For simplicity of presentation, we omit the projection operator in the figure labels. Depending on $\thetatrue$, the parameter error of $\augest$ along $e_2$ could be larger or smaller than the parameter error of $\stdest$ along $e_2$. 
      \textbf{(c)--(d) Dependence of space of safe augmentations on $\sigmapop$.}
      Visualization of the space of extra data points $\singleaug$ (orange), that do not cause an increase in the standard
      error for the illustrated $\thetatrue$ (black vector), as result
      of Theorem~\ref{thm:main}.
  }
    \label{fig:simpleexistence}

\end{figure}

\subsection{General characterizations}
\label{sec:general}
In this section, we precisely characterize when the augmented estimator $\augest$ that fits extra training data points $\xext$ in addition to the standard points $\xstd$ has higher standard error than the standard estimator $\stdest$ that only fits $\xstd$. In particular, this enables us to understand when there is a ``tradeoff'' where the augmented estimator $\augest$ has lower robust error than $\stdest$ by virtue of fitting perturbations, but has higher standard error. In Section~\ref{sec:simple-3D}, we illustrated how the parameter error of $\augest$ could be larger than $\stdest$ in some directions, and if these directions are weighted heavily in the population covariance $\sigmapop$, the standard error of $\augest$ would be larger.

Formally, let us define the parameter errors $\deltastd \eqdef \stdest -
\thetatrue$ and $\deltaaug \eqdef \augest -\thetatrue$. Recall that the standard errors are
\begin{align}
  \label{eqn:errors}
\stderr(\stdest) = \deltastd^\top \sigmapop \deltastd, ~~\stderr(\augest) = \deltaaug^\top \sigmapop \deltaaug, 
\end{align}
where $\sigmapop$ is the population covariance of the underlying inputs drawn from $\distribx$.

To characterize the effect of the inductive bias of minimum norm interpolation on the standard errors, we define the following projection operators: $\pistd$, the projection matrix onto $\Null(\xstd)$ and $\piaug$, the projection matrix onto $\Null([\xext;\xstd])$ (see formal definition in Appendix~\ref{sec:app-bias}). Since $\augest$ and $\stdest$ are minimum norm interpolants, $\pistd \stdest = 0$ and $\piaug \augest = 0$. Further, in noiseless linear regression, $\stdest$ and $\augest$ have no error in the span of $\xstd$ and $[\xstd; \xext]$ respectively. Hence, 
\begin{align}
  \label{eqn:param-errors}
  \deltastd = \pistd \thetatrue, ~~ \deltaaug = \piaug \thetatrue.
\end{align}
Our main result relies on the key observation that for any vector $u$, $\pistd
u$ can be decomposed into a sum of two orthogonal components $v$ and $w$ such that
$\pistd u = v + w$ with $w= \piaug u$ and $v = \pistd \piparaaug u$. This is because $\Null([\xstd; \xext]) \subseteq \Null(\xstd)$ and thus $\pistd \piaug = \piaug$.
  Now setting $u=\thetatrue$ and using the error expressions in Equation~\ref{eqn:errors}
and Equation~\ref{eqn:param-errors} gives a precise characterization of the difference in the standard
errors of $\stdest$ and $\augest$.
\begin{theorem}
  \label{thm:main}
The difference in the standard errors of the standard estimator $\stdest$ and augmented estimator $\augest$ can be written as follows. 
\begin{align}
    \label{eqn:exactwv}
  \stderr(\stdest) -\stderr(\augest) &= v^\top \sigmapop v + 2 w^\top \sigmapop v, 
\end{align}
where $v = \pistd \piparaaug \thetatrue$ and $w = \piaug \thetatrue$.
\end{theorem}
The proof of Theorem~\ref{thm:main} is in
Appendix~\ref{sec:app-existltwo}. The increase in standard error of the augmented estimator can be understood in terms of the vectors $w$ and $v$ defined in Theorem~\ref{thm:main}. The first term $v^\top \sigmapop v$ is always
positive, and corresponds to the decrease in the standard error of the augmented
estimator $\augest$ by virtue of fitting extra training points in some directions. However, the second term $2 w^\top \sigmapop v$ can be negative and intuitively measures the cost of a possible increase in the parameter error along other directions (similar to the increase along $e_2$ in the simple setting of Figure~\ref{fig:simpleexistence}(a)). When the cost outweighs the benefit, the standard error of $\augest$ is larger. Note that both the cost and benefit is determined by $\sigmapop$ which governs how the parameter error affects the standard error. 

We can use the above expression (Theorem~\ref{thm:main}) for the difference in standard errors of $\augest$ and $\stdest$ to characterize different ``safe'' conditions under which augmentation with extra data does not increase the standard error. See Appendix~\ref{sec:characterization} for a proof. 
\begin{corollary}
  \label{thm:cor}
  The following conditions are sufficient for $\stderr(\augest) \leq \stderr(\stdest)$, i.e.
  the standard error does not increase when fitting augmented data.
\begin{enumerate}
\item The population covariance $\sigmapop$ is identity.
\item The augmented data $[\xstd; \xext]$ spans the entire space, or equivalently $\piaug = 0$.
\item The extra data $\singleaug \in \R^d$ is a single point such that $\singleaug$ is an eigenvector of $\sigmapop$.
\end{enumerate}
\end{corollary}

\paragraph{Matching inductive bias.}
We would like to draw special attention to the first condition. When
$\sigmapop = I$, notice that the norm that governs the standard error
(Equation~\ref{eqn:errors}) matches the norm that is minimized by the
interpolants (Equation~\ref{eqn:est}). Intuitively, the estimators
have the ``right'' inductive bias; under this condition, the augmented
estimator $\augest$ does not have higher standard error. In other
words, the observed increase in the standard error of $\augest$ can be
attributed to the ``wrong'' inductive bias. In Section~\ref{sec:rst},
we will use this understanding to propose a method of robust training
which does not increase standard error over standard training.

\paragraph{Safe extra points.} We use Theorem~\ref{thm:main} to plot the safe extra points $\singleaug \in \R^d$ that do not lead to an increase in standard error for any $\thetatrue$ in the simple 3D setting described in Section~\ref{sec:simple-3D} for two different $\sigmapop$ (Figure~\ref{fig:simpleexistence} (c), (d)). 
The safe points lie in cones which contain the eigenvectors of $\sigmapop$ (as expected from Corollary~\ref{thm:cor}). The width and alignment of the cones depends on the alignment between $\thetatrue$ and the eigenvectors of $\sigmapop$. As the eigenvalues of $\sigmapop$ become less skewed, the space of safe points expands, eventually covering the entire space when $\sigmapop = I$ (see Corollary~\ref{thm:cor}). 

\paragraph{Local versus global structure.}
We now tie our analysis back to the cubic splines interpolation problem from Figure~\ref{fig:spline}. The inputs can be appropriately rotated and scaled such that the cubic spline interpolant is the minimum Euclidean norm interpolant (as in Equation~\ref{eqn:est}). Under this transformation, the different eigenvectors of the nullspace of the training data $\Null(\xstd)$ represent the ``local'' high frequency components with small eigenvalues or ``global'' low frequency components with large eigenvalues (see Figure~\ref{fig:K_eigenvectors}). An augmentation that encourages the fitting local components in $\Null(\xstd)$ could potentially increase the error along other global components (like the increase in error along $e_2$ in Figure~\ref{fig:simpleexistence}(a)). Such an increase, coupled with the fact that global components have larger eigenvalue in $\sigmapop$, results in the standard error of $\augest$ being larger than that of $\stdest$. See Figure~\ref{fig:K_eigenvectors_2} and Appendix~\ref{sec:splines-local-global} for more details. This is similar to the recent observation that adversarial training with $\ell_\infty$ perturbations encourages neural networks to fit the high frequency components of the signal while compromising on the low-frequency components~\citep{yin2019fourier}. 

\paragraph{Model complexity.} Finally, we relate the magnitude of increase in standard error of the augmented estimator to the complexity of the true model.
\begin{proposition}
  \label{prop:simple-complex}
  For a given $\xstd, \xext, \sigmapop$, 
\begin{align*}
\stderr(\augest) - \stderr(\stdest) > c \implies \|\thetatrue\|^2_2 - \| \stdest \|^2_2 > \gamma c
\end{align*} 
for some scalar $\gamma>0$ that depends on $\xstd, \xext, \sigmapop$.
\end{proposition}
In other words, for a large increase in standard error upon augmentation, the true parameter $\thetatrue$ needs to be sufficiently more complex (in the $\ell_2$ norm) than the standard estimator $\stdest$. For example, the construction of the cubic splines interpolation problem relies on the
underlying function (staircase) being more complex with additional
local structure than the standard estimator---a linear function that
fits most points and can be learned with few
samples. Proposition~\ref{prop:simple-complex} states that this
requirement holds more generally. The proof of
Proposition~\ref{prop:simple-complex} appears in
Appendix~\ref{sec:app-lowerbound}.
A similar intuition can be used to construct an example where augmentation can increase standard error for minimum $\ell_1$-norm interpolants when $\theta^\star$ is dense (Appendix~\ref{app:l1_problem}).

\begin{figure}[t]
  \centering
      \begin{minipage}[c]{0.24\linewidth}
        \includegraphics[scale=0.14]{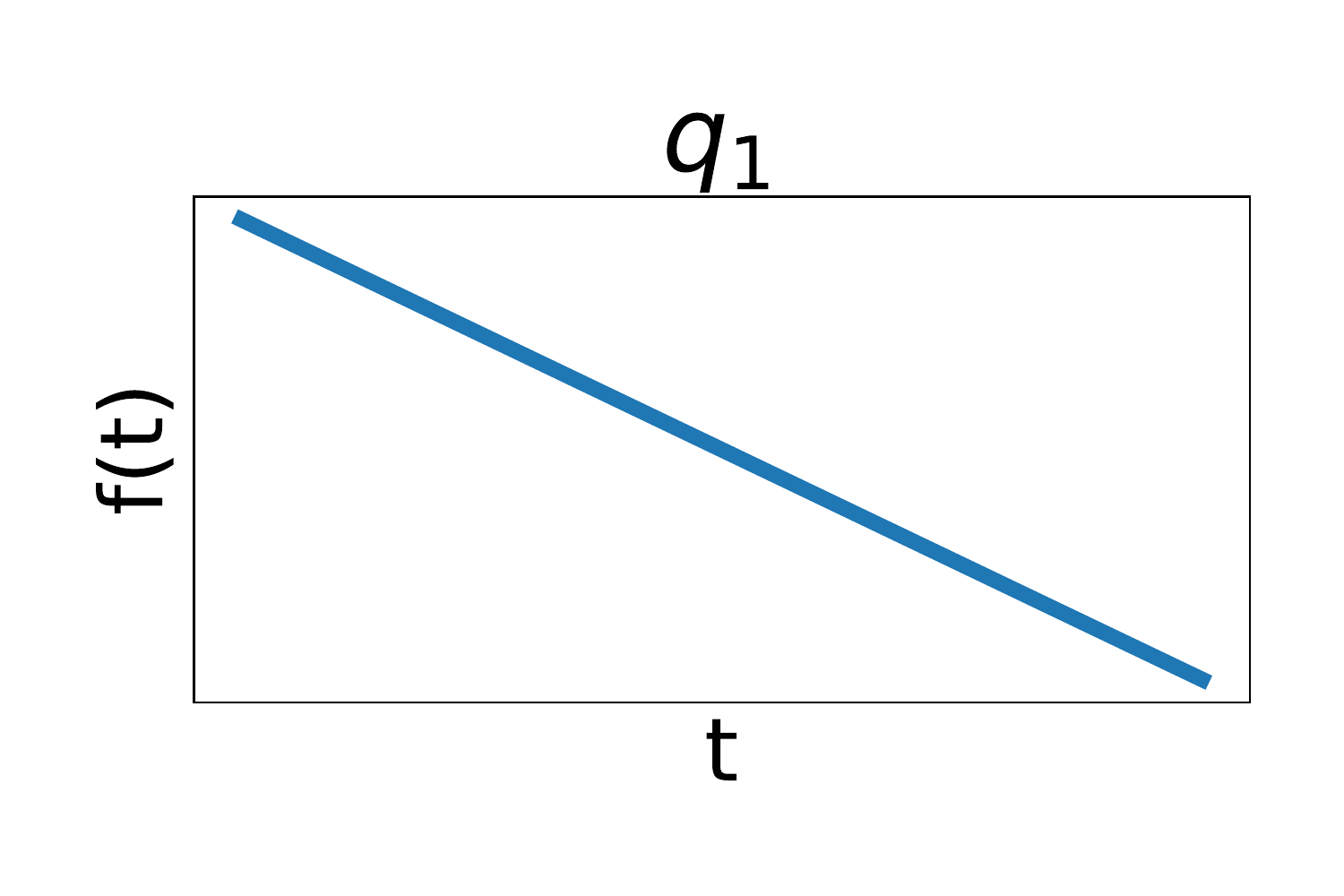}
      \end{minipage}
      \begin{minipage}[c]{0.24\linewidth}
        \includegraphics[scale=0.14]{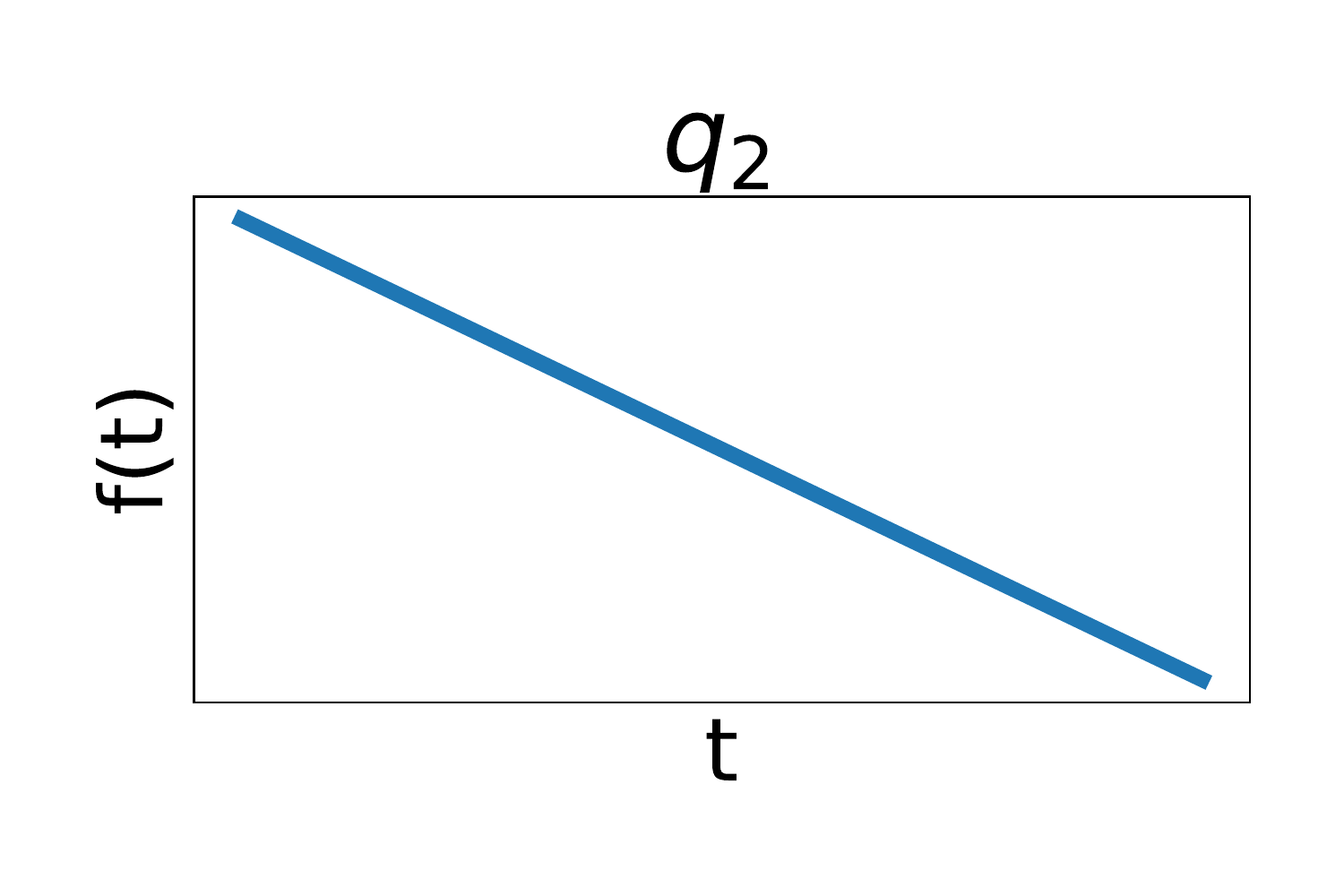}
      \end{minipage}
      \begin{minipage}[c]{0.24\linewidth}
        \includegraphics[scale=0.14]{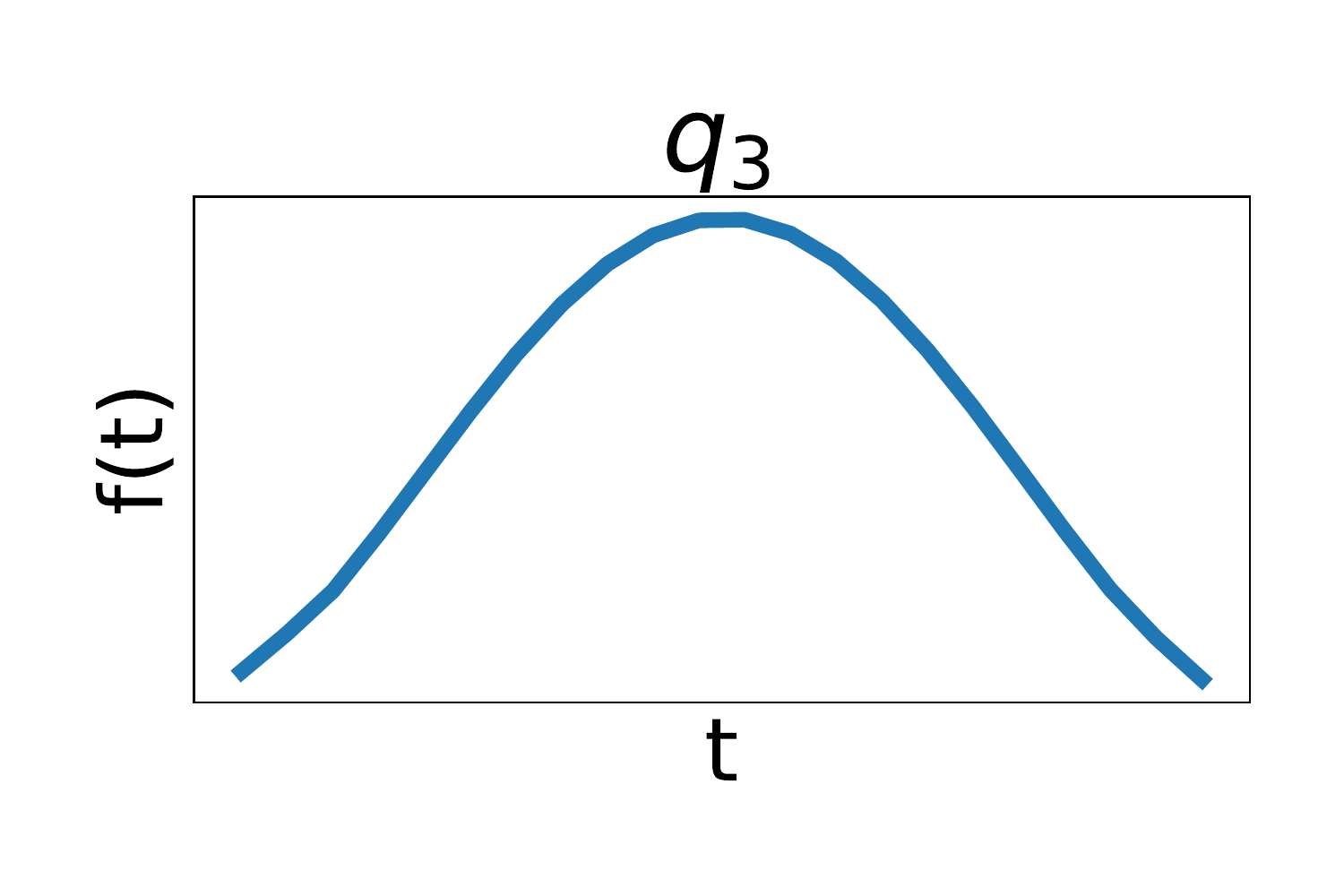}
      \end{minipage}
      \begin{minipage}[c]{0.24\linewidth}
        \includegraphics[scale=0.14]{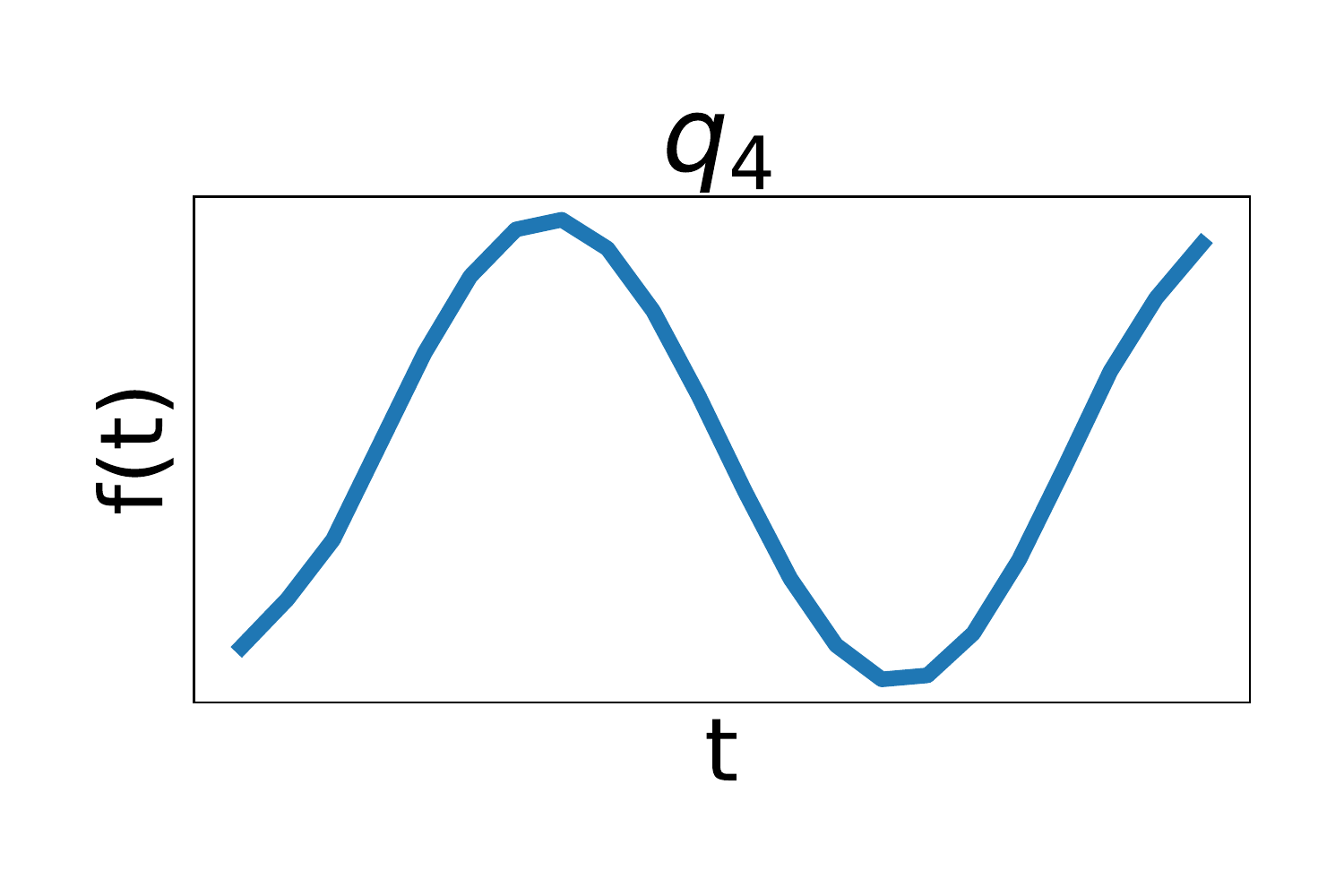}
      \end{minipage}
      \caption{Top 4 eigenvectors of $\sigmapop$ in the splines problem (from Figure~\ref{fig:spline}), representing wave functions in the input space. The ``global'' eigenfunctions, varying less over the domain, correspond to larger eigenvalues, making errors in global dimensions costly in terms of test error.  }
      \label{fig:K_eigenvectors}
\end{figure}

\section{Robust self-training}
\label{sec:rst}
We now use insights from Section~\ref{sec:min-norm}
to construct estimators with low robust error without increasing the
standard error. While Section~\ref{sec:min-norm} characterized the effect of
adding extra data $\xext$ in general, in this section
we consider robust training which augments the dataset with extra data $\xext$ that are consistent perturbations of the standard training data $\xstd$. 

Since the standard estimator has small standard error, a natural strategy
to mitigate the tradeoff is to regularize the augmented estimator
to be closer to the standard estimator. The choice of distance between the estimators we regularize is very important.
Recall from Section~\ref{sec:simple-3D} that the population covariance $\sigmapop$ determines how the parameter error affects the standard error. This suggests using a regularizer that incorporates information about $\sigmapop$. 

We first revisit the recently proposed robust self-training (RST)~\citep{carmon2019unlabeled,
  najafi2019robustness, uesato2019are} that incorporates additional unlabeled data via pseudo-labels from a standard estimator. Previous work only focused on the effectiveness of RST in improving the robust error. In Section~\ref{sec:linear-rst}, we prove that in linear regression, RST eliminates the tradeoff between standard and robust error (Theorem~\ref{thm:linear-x}). The proof hinges on the connection between RST and the idea of regularizing towards the standard estimator discussed above. In particular, we show that the RST objective can be rewritten as minimizing a suitable $\sigmapop$-induced distance to the standard estimator. 

In Section~\ref{sec:empirical-rst}, we expand upon
previous empirical RST results for \cifar~across various training set sizes and perturbations (rotations/translations in addition to $\ell_\infty$).
We observe that across all settings, RST substantially improves the standard error while also
improving the robust error over the vanilla supervised robust training counterparts. 

\subsection{General formulation of RST}
We first describe the general two-step robust self-training (RST) procedure~\citep{carmon2019unlabeled, uesato2019are} for a parameteric model $f_\theta$:
\begin{enumerate}
  \item Perform standard training on labeled data $\{(x_i, y_i)\}_{i=1}^n$ to obtain
      $\stdest = \argmin_{\theta} \sum\limits_{i=1}^n \ell \big(f_\theta(x_i), y_i)$.
\item Perform robust training on both the labeled data and unlabeled inputs $\{\tilde{x}_i\}_{i=1}^m$ with \emph{pseudo-labels} $\tilde{y}_{i} = f_{\stdest}(\tilde{x}_i)$ generated from the standard estimator $\stdest$.
\end{enumerate}
The second stage typically involves a combination of the standard loss $\ell$ and a robust loss $\ellrob$. The robust loss encourages invariance of the model over perturbations $T(x)$, and is generally defined as
\begin{equation}
  \label{eqn:deflrob}
  \ellrob(f_\theta(x_i),y_i) = \max_{x_{\text{adv}} \in T(x_i)} \ell(f_\theta(x_{\text{adv}}), y_i).
\end{equation}
It is convenient to summarize the robust self-training estimator $\thetarst$ as the minimizer of a weighted combination of four separate losses as follows. 
We define the losses on the labeled dataset $\{(x_i, y_i)\}_{i=1}^n$ as
\begin{align*}
    \Lsl(\theta) &= \frac{1}{n}\sum_{i=1}^n \ell(f_\theta(x_i), y_i), \\
    \Lrl(\theta) &= \frac{1}{n}\sum_{i=1}^n \ellrob(f_\theta(x_i), y_i).
\end{align*}
The losses on the unlabeled samples $\{\tilde{x}_i\}_{i=1}^m$ which are psuedo-labeled by the standard estimator are
\begin{align*}
    \Lsu(\theta; \stdest) &= \frac{1}{m}\sum_{i=1}^m\ell(f_\theta(\tilde{x}_i), \pseudoy{i}), \\
    \Lru(\theta; \stdest) &= \frac{1}{m}\sum_{i=1}^m \ellrob(f_\theta(\tilde{x}_i), \pseudoy{i}).
\end{align*}
Putting it all together, we have
\begin{align}
    \thetarst \coloneqq \argmin_{\theta}\Big(&\alpha\Lsl(\theta) + \beta \Lrl(\theta)  \label{eqn:general-x}\\ &+ \gamma \Lsu(\theta; \stdest) + \lambda \Lru(\theta; \stdest)\Big) \nonumber,
\end{align}
for fixed scalars $\alpha, \beta, \gamma, \lambda \geq 0$.

\subsection{Robust self-training for linear regression}
\label{sec:linear-rst}
\begin{figure}
  \center
  \includegraphics[scale=0.17]{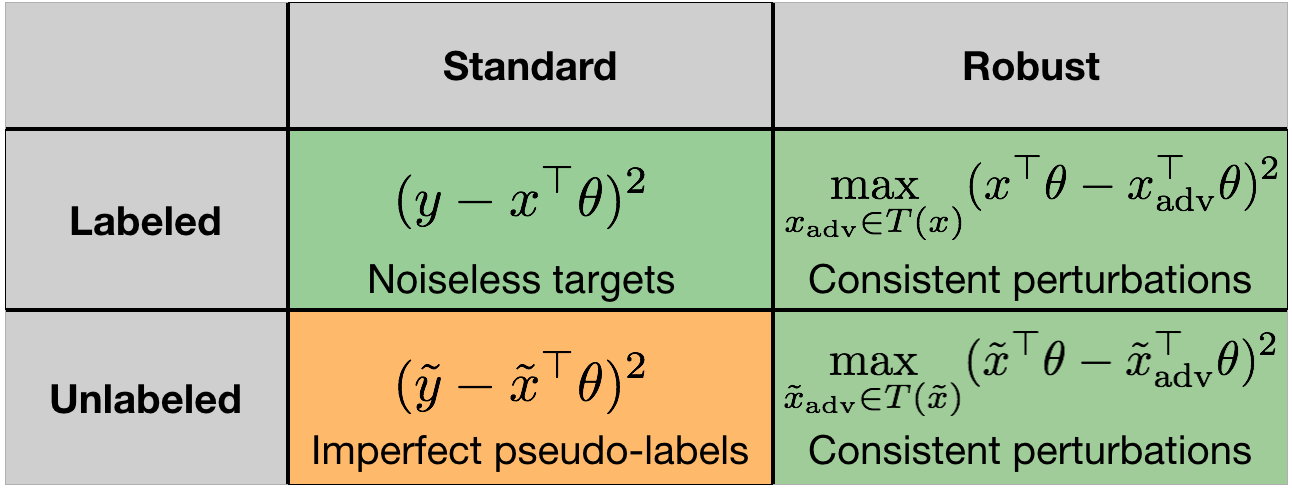}
  \caption{Illustration shows the four components of the RST loss (Equation~\eqref{eqn:general-x}) in the special case of linear regression (Eq.~\eqref{eqn:linear-x}). Green cells contain hard constraints where the optimal $\theta^\star$ obtains zero loss. The orange cell contains the soft constraint that is minimized while satisfying hard constraints to obtain the final linear RST estimator. }
  \label{fig:matrix}
\end{figure}

We now return to the noiseless linear regression as
described in Section~\ref{sec:setting} and specialize the general RST estimator described in Equation \refeqn{general-x} to this setting. We prove that RST eliminates the decrease in standard error in this setting while achieving low robust error by showing that RST appropriately regularizes the augmented estimator towards the standard estimator.  

Our theoretical results hold for RST procedures where the pseudo-labels can be generated from any interpolating estimator $\thetainterp$ satisfying $\xstd \thetainterp = \ystd$. This includes but is not restricted to the mininum-norm standard estimator $\stdest$ defined in \eqref{eqn:est}. We use the squared loss as the loss function $\ell$. 
For consistent perturbations $T(\cdot)$, we analyze the following RST estimator for linear regression
\begin{align}
  \label{eqn:linear-x}
  \thetarst = \argmin_{\theta}\{&\pLsu(\theta; \thetainterp): \pLru(\theta)=0 , \nonumber\\
  &\Lsl(\theta) = 0, \Lrl(\theta) = 0\}.
\end{align}
Figure~\ref{fig:matrix} shows the four losses of RST in this special case of linear regression.

Obtaining this specialized estimator from the general RST estimator in Equation \refeqn{general-x} involves the following steps. First, for convenience of analysis, we assume access to the population covariance $\sigmapop$ via infinite unlabeled data and thus replace the finite sample losses on the unlabeled data $\Lsu(\theta), \Lru(\theta)$ by their population losses $\pLsu(\theta), \pLru(\theta)$. Second, the general RST objective minimizes some weighted combination of four losses. When specializing to the case of noiseless linear regression, since $\hat{L}_{\text{std, lab}}(\thetatrue) = 0$, rather than minimizing $\alpha \Lsl(\thetatrue)$, we set the coefficients on the losses such that the estimator satisfies a hard constraint $\Lsl(\thetatrue)=0$. This constraint which enforces interpolation on the labeled dataset $y_i = x_i^\top \theta~\forall i = 1, \hdots n$ allows us to rewrite the robust loss (Equation~\ref{eqn:deflrob}) on the labeled examples equivalently as a self-consistency loss defined independent of labels. 
\begin{align*}
  \Lrl(\theta)&= \frac{1}{n}\sum_{i=1}^n \max_{x_{\text{adv}} \in T(x)} (x_i^\top \theta - x_{\text{adv}}^\top \theta)^2.
\end{align*}
Since $\thetatrue$ is invariant on perturbations $T(x)$ by definition, we have $\Lrl(\thetatrue) = 0$ and thus we introduce a constraint $\Lrl(\theta)=0$ in the estimator. 

For the losses on the unlabeled data, since the pseudo-labels are not perfect, we minimize $\pLsu$ in the objective instead of enforcing a hard constraint on $\pLsu$. However, similarly to the robust loss on labeled data, we can reformulate the robust loss on unlabeled samples $\pLru$ as a self-consistency loss that does not use pseudo-labels. By definition, $\pLru(\thetatrue) = 0$ and thus we enforce $\pLru(\theta) = 0$ in the specialized estimator. 

We now study the standard and robust error of the linear regression RST estimator defined above in Equation \refeqn{linear-x}. 
\begin{restatable}[]{theorem}{linearx}
  \label{thm:linear-x}
  Assume the noiseless linear model $y = x^\top \thetatrue$.
  Let $\thetainterp$ be an arbitrary interpolant of the standard data, i.e.\ $\xstd \thetainterp = \ystd$.
  Then
  \begin{align*}
    \stderr\big(\thetarst) \leq \stderr(\thetainterp).
  \end{align*}
  Simultaneously, $\roberr(\thetarst) = \stderr(\thetarst)$. 
\end{restatable}
See Appendix~\ref{sec:app-rst} for a full proof.

The crux of the proof is that the optimization objective of RST is an inductive bias that regularizes the estimator to be close to the standard estimator, weighing directions by their contribution to the standard error via $\sigmapop$.
To see this, we rewrite
\begin{align*}
    \pLsu(\theta; \thetainterp) &= \E_{\distribx} [(\tilde{x}^\top \thetainterp - \tilde{x}^\top \theta)^2] \\
    &= (\thetainterp - \theta)^\top \sigmapop (\thetainterp - \theta).
\end{align*}
By incorporating an appropriate $\sigmapop$-induced regularizer while satisfying constraints on the robust losses, RST ensures that the standard error of the estimator never exceeds the standard error of $\stdest$. The robust error of any estimator is lower bounded by its standard error, and this gap can be arbitrarily large for the standard estimator. However, the robust error of the RST estimator matches the lower bound of its standard error which in turn is bounded by the standard error of the standard estimator and hence is small. To provide some graphical intuition for the result, see Figure~\ref{fig:spline} that visualizes the RST estimator on the cubic splines interpolation problem that exemplifies the increase in standard error upon augmentation. RST captures the global
structure and obtains low standard error by matching $\stdest$ (straight line) on unlabeled inputs. Simultaneously, RST enforces invariance on local transformations on both labeled and unlabeled inputs, and obtains low robust error by capturing the local structure across the domain.

\paragraph{Implementation of linear RST.} The constraint on the standard loss on labeled data simply corresponds to interpolation on the standard labeled data. The constraints on the robust self-consistency losses involve a maximization over a set of transformations. In the case of linear regression, such constraints can be equivalently represented by a set of at most $d$ linear constraints, where $d$ is the dimension of the covariates. Further, with this finite set of constraints, we only require access to the covariance $\sigmapop$ in order to constrain the population robust loss. Appendix~\ref{sec:app-rst} gives a practical iterative algorithm that computes the RST estimator for linear regression reminiscent of adversarial training in the semi-supervised setting.
\begin{table*}
\parbox{.49\textwidth}{
\resizebox{0.8\linewidth}{!}{%
  \begin{tabular} {>{\raggedright}p{4cm} | p{1.5cm} | p{1.5cm}}
    \toprule
    Method & Robust Test Acc. & Standard Test Acc.\\
    \midrule
    Standard Training & 0.8\% & 95.2\% \rdelim\}{3}{3mm}[\large $\substack{\Huge \text{Vanilla}\\ \text{Supervised}}$]\\
    PG-AT~\citep{madry2018towards} & 45.8\% & 87.3\% \\
    TRADES~\citep{zhang2019theoretically} & 55.4\% & 84.0\% \\
    
    \cellcolor{black!15} Standard Self-Training & 0.3\% & 96.4\% \rdelim\}{5}{3mm}[\large $\substack{\text{Semisupervised} \\ \text{with same} \\\text{unlabeled data}}$]\\
    \cellcolor{black!15} Robust Consistency Training\linebreak~\citep{carmon2019unlabeled}&56.5\% & 83.2\% \\
    \cellcolor{black!15} \textbf{RST + PG-AT (this paper)} & \textbf{58.5\%} & \textbf{91.8\%} \\
    \cellcolor{black!15} \textbf{RST + TRADES (this paper)}\linebreak~\citep{carmon2019unlabeled} & \textbf{63.1\%} & \textbf{89.7\%} \\
    Interpolated AT\linebreak~\citep{lamb2019interpolated}\footnote{Used a slightly
      smaller WRN-20-10 model} & 45.1\% & 93.6\%  \rdelim\}{4}{3mm}[\large $\substack{\text{Modified} \\ \text{supervised}}$]\\
    Neural Arch. Search\linebreak~\citep{cubuk2017intriguing} & 50.1\% & 93.2\% \\
    \bottomrule
  \end{tabular}
}
}
\hfill
\parbox{.49\textwidth} {
  \resizebox{0.8\linewidth}{!}{%
    \begin{tabular} {>{\raggedright}p{4cm} | p{1.5cm} | p{1.5cm}}
      \toprule
    Method & Robust Test Acc. & Standard Test Acc. \\
    \midrule
      Standard Training & 0.2\% & 94.6\% \rdelim\}{3}{3mm}[\large $\substack{\text{Vanilla}\\ \text{Supervised}}$]\\
      Worst-of-10 & 73.9\% & 95.0\% \\
      Random & 67.7\% & 95.1\% \\
  \cellcolor{black!15} \textbf{RST + Worst-of-10 (this paper)} & \textbf{75.1\%} & \textbf{95.8\%} \rdelim\}{2}{3mm}[\large $\substack{\text{Semisupervised}}$]\\
  \cellcolor{black!15} \textbf{RST + Random (this paper)} & \textbf{70.9\%} & \textbf{95.8\%} \\
      Worst-of-10\linebreak~\citep{engstrom2019exploring}\footnote{Used a smaller ResNet model} & 69.2\% & 91.3\% \rdelim\}{3}{3mm}[\large $\substack{\text{Existing baselines} \\ \text{(smaller model)}}$]\\
      Random~\citep{yang2019invariance}\footnote{Used a smaller ResNet-32 model} & 58.3\% & 91.8\% \\
    \bottomrule
  \end{tabular}
}
}
\caption{
    Performance of robust self-training (RST) applied to different perturbations and adversarial training algorithms.
    \textbf{(Left)} \cifar~ standard and robust test accuracy against $\ell_\infty$ perturbations of size $\epsilon=8/255$. All methods use $\epsilon=8/255$ while training and use the WRN-28-10 model. Robust accuracies are against a PG based attack with 20 steps.
    \textbf{(Right)} \cifar~ standard and robust test accuracy against a grid attack of rotations up to 30 degrees and translations up to $\sim10\%$ of the image size, following~\citep{engstrom2019exploring}. All adversarial and random methods use the same parameters during training and use the WRN-40-2 model. 
    For both tables, shaded rows make use of 500K unlabeled images from 80M Tiny Images sourced in~\citep{carmon2019unlabeled}. RST improves \emph{both} the standard and robust accuracy over the vanilla counterparts for different algorithms (AT and TRADES) and different perturbations ($\ell_\infty$ and rotation/translations). 
}
\label{table:adv-results}
\end{table*}

\begin{figure}[tbp]
  \centering
  \begin{subfigure}{0.45\linewidth} 
    \includegraphics[scale=0.25]{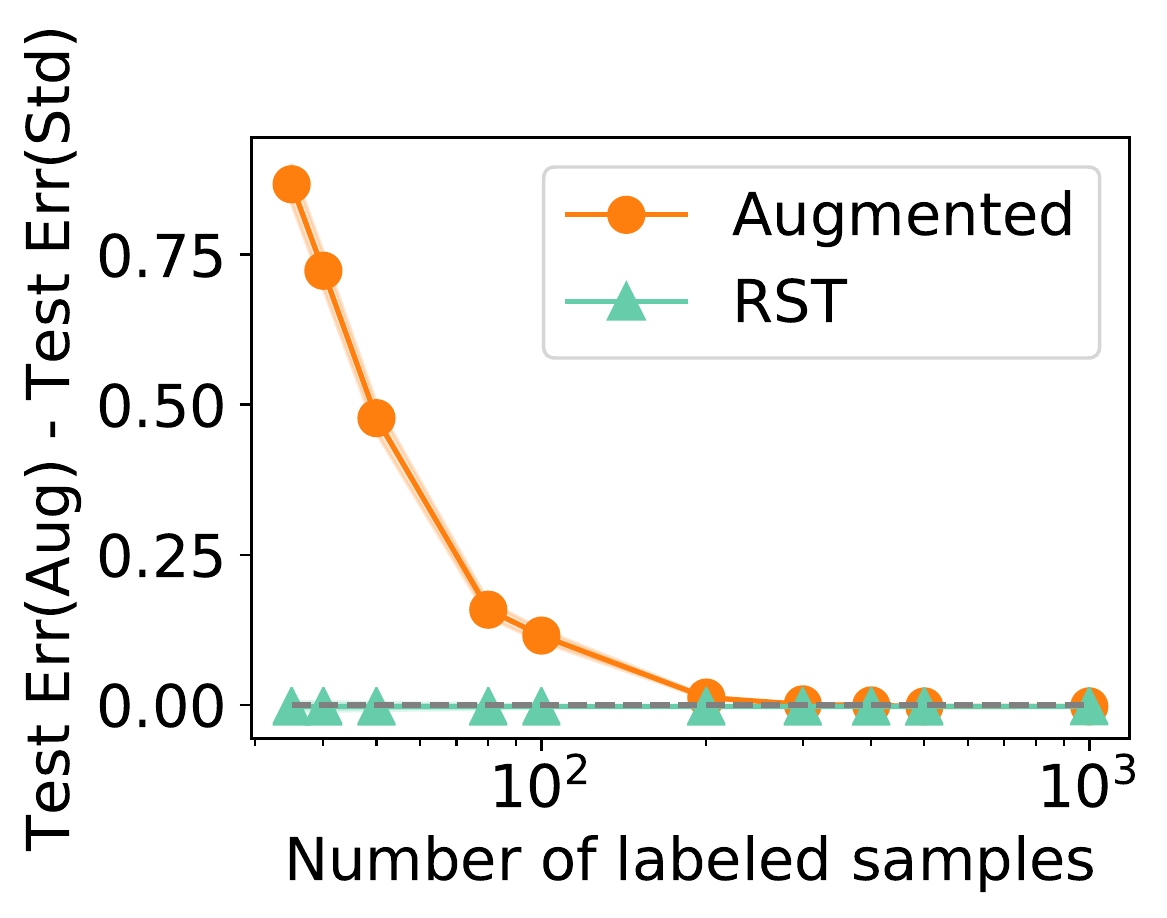}
    \caption{Spline Staircase}
  \end{subfigure}
  \begin{subfigure}{0.45\linewidth}
    \includegraphics[scale=0.22]{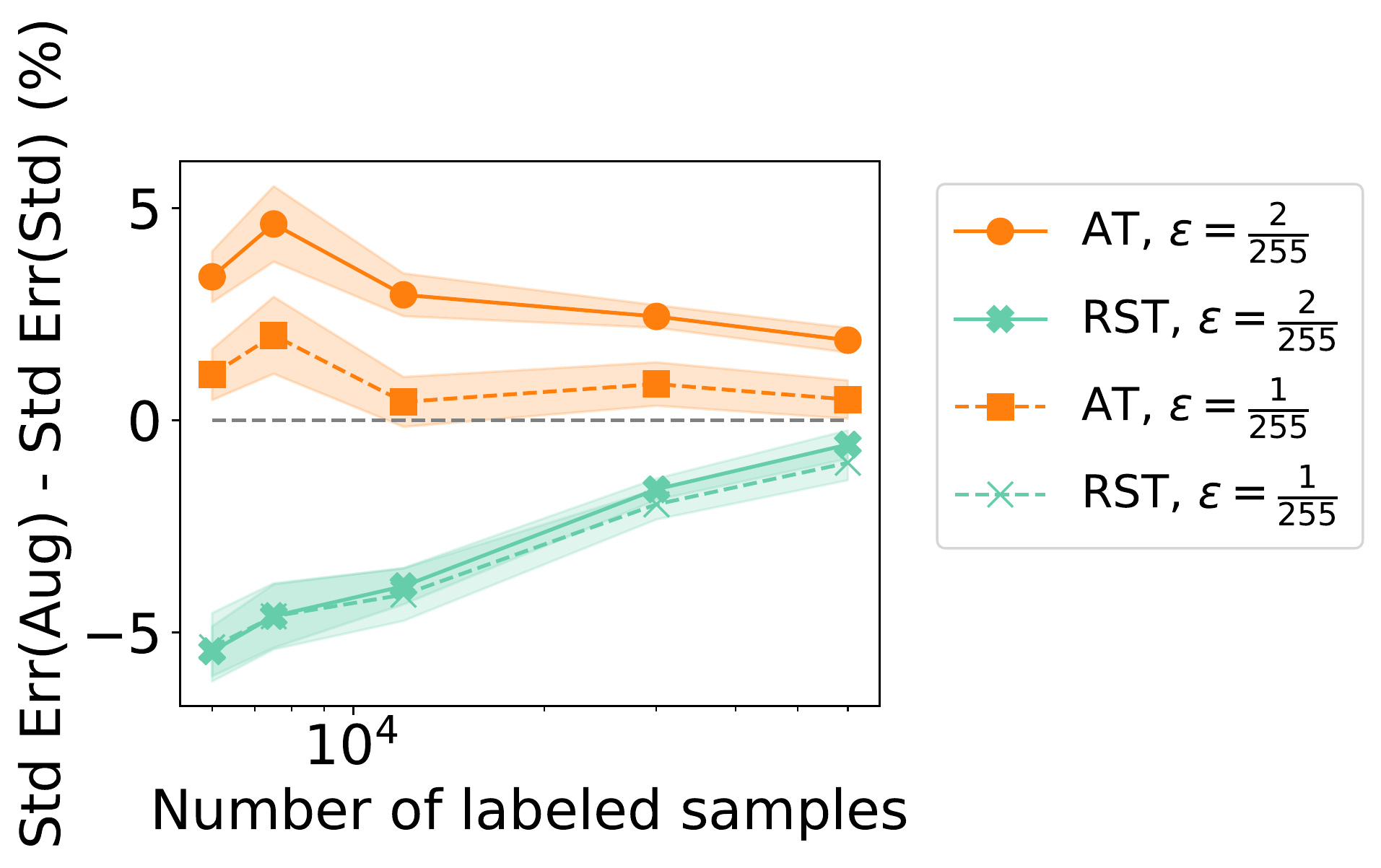}
    \caption{\cifar~(AT)}
  \end{subfigure}
  \caption{Effect of data augmentation on test error as we vary the number of training samples.
  \textbf{(a)-(b)} We plot the difference in errors of the augmented estimator and standard estimator. In both the spline staircase simulations and data augmentation with adversarial $\ell_\infty$ perturbations via adversarial training (AT) on~\cifar, the increase in test error decreases as the training sample size increases. In \textbf{(b)}, robust self-training (RST+AT) not only mitigates the increase in test error from AT but even improves test error beyond that of the standard estimator.}
  \label{fig:sample-size-rst}
\end{figure}

\subsection{Empirical evaluation of RST}
\label{sec:empirical-rst}
\citet{carmon2019unlabeled} empirically evaluate RST with a focus on studying gains in the robust error. In this work, we focus on \emph{both} the standard and robust error and expand upon results from previous work. \citet{carmon2019unlabeled} used TRADES~\citep{zhang2019theoretically} as the robust loss in the general RST formulation~\refeqn{general-x}; we additionally evaluate RST with Projected Gradient Adversarial Training (AT)~\citep{madry2018towards} as the robust loss. ~\citet{carmon2019unlabeled} considered $\ell_\infty$ and $\ell_2$ perturbations. We study rotations and translations in addition to $\ell_\infty$ perturbations, and also study the effect of labeled training set size on standard and robust error. Table~\ref{table:adv-results} presents the main results. More experiment details appear in Appendix~\ref{sec:app-xreg-robustness}. 

Both RST+AT and RST+TRADES have lower robust and standard error than their supervised counterparts AT and TRADES across all perturbation types. This mirrors the theoretical analysis of RST in linear regression (Theorem~\ref{thm:linear-x}) where the RST estimator has small robust error while provably not sacrificing standard error, and never obtaining larger standard error than the standard estimator.  

\paragraph{Effect of labeled sample size.}
Recall that our work motivates studying the tradeoff between robust and standard error while taking \emph{generalization} from finite data into account. We showed that the gap in the standard error of a standard estimator and that of a robust estimator is large for small training set sizes and decreases as the labeled dataset is larger (Figure~\ref{fig:sample-size}). We now study the effect of RST as we vary the training set size in Figure~\ref{fig:sample-size-rst}. We find that RST+AT has \emph{lower} standard error than standard training across all sample sizes for small $\epsilon$, while simultaneously achieving lower robust error than AT (see Appendix~\ref{app:subsample-cifar}). In the small data regime where vanilla adversarial training hurts the standard error the most, we find that RST+AT gives about 3x more absolute improvement than in the large data regime. We note that this set of experiments are complementary to the experiments in~\citep{schmidt2018adversarially} which study the effect of the training set size only on robust error. 

\paragraph{Effect on transformations that do not hurt standard error.}
We also test the effect of RST on perturbations where robust training slightly improves standard error rather than hurting it. 
Since RST regularizes towards the standard estimator, one might suspect that the improvements from robust training disappear with RST.
In particular, we consider spatial transformations $T(x)$ that consist of simultaneous rotations and translations.
We use two common forms of robust training for spatial perturbations, where we approximately maximize over $T(x)$ with either adversarial (worst-of-10) or random augmentations~\citep{yang2019invariance, engstrom2019exploring}. Table~\ref{table:adv-results} (right) presents the results.
In the regime where vanilla robust training does not hurt standard error, RST in fact further improves the standard error by almost 1\% and the robust error by 2-3\% over the standard and robust estimators for both forms of robust training. Thus in settings where vanilla robust training improves standard error, RST seems to further amplify the gains while in settings where vanilla robust training hurts standard error, RST mitigates the harmful effect.

\paragraph{Comparison to other semi-supervised approaches.}
The RST estimator minimizes both a robust loss and a standard loss on the unlabeled data with pseudo-labels (bottom row, Figure~\ref{fig:matrix}). Both of these losses are necessary to simultaneously the standard and robust error over vanilla supervised robust training. Standard self-training, which only uses standard loss on unlabeled data, has very high robust error ($\approx 100\%$). Similarly, Robust Consistency Training, an extension of Virtual Adversarial Training~\citep{miyato2018virtual} that only minimizes a robust self-consistency loss on unlabeled data, marginally improves the robust error but actually \emph{hurts} standard error (Table~\ref{table:adv-results}). 

\section{Related Work}
\paragraph{Existence of a tradeoff.} Several works have attempted to explain the tradeoff between standard and robust error by studying simple models. These explanations are based on an \emph{inherent tradeoff} that persists even in the infinite data limit. In ~\citet{tsipras2019robustness, zhang2019theoretically, fawzi2018analysis}, standard and robust error are fundamentally at odds, meaning no classifier is both accurate and robust. In \citet{nakkiran2019adversarial}, the tradeoff is due to the hypothesis class not being expressive enough to contain an accurate and robust classifier even if it exists. In contrast, we explain the tradeoff in a more realistic setting with label-preserving consistent perturbations (like imperceptible $\ell_\infty$ perturbations or small rotations) in a well-specified setting (to mirror expressive neural networks) where there is no tradeoff with infinite data. In particular, our work takes into account generalization from finite data to explain the tradeoff.

In concurrent and independent work,~\citet{min2020curious} also study the effect of dataset size on the tradeoff. They prove that in a ``strong adversary'' regime, there is a tradeoff even with infinite data, as the perturbations are large enough to change the ground truth target. They also identify a ``weak adversary'' regime (smaller perturbations) where the gap in standard error between robust and standard estimators first increases and then decreases, with no tradeoff in the infinite data limit. Similar to our work, this provides an example of a tradeoff due to generalization from finite data. However, their experimental validation of the tradeoff trends is restricted to simulated settings and they do not study how to mitigate the tradeoff.

\paragraph{Mitigating the tradeoff.} To the best of our knowledge, ours is the first work that theoretically studies how to mitigate the tradeoff between standard and robust error. While robust self-training (RST) was proposed in recent works~\citep{carmon2019unlabeled, najafi2019robustness, uesato2019are} as a way to improve \emph{robust error}, we prove that RST eliminates the tradeoff between standard and robust error in noiseless linear regression and systematically study the effect on RST on the tradeoff with several different perturbations and adversarial training algorithms on~\cifar.

Interpolated Adversarial Training (IAT)~\citep{lamb2019interpolated} and Neural Architecture Search (NAS)~\citep{cubuk2017intriguing} were proposed to mitigate the tradeoff bbetween standard and robust error empirically. IAS considers a different training algorithm based on Mixup, NAS~\citep{cubuk2017intriguing} uses RL to search for more robust architectures. In Table~\ref{table:adv-results}, we also report the standard and robust errors of these methods. RST, IAT and NAS are incomparable as they find different tradeoffs between standard and robust error. Recently,~\citet{xie2020adversarial} showed that adversarial training with appropriate batch normalization (AdvProp) with small perturbations can actually \emph{improve} standard error. However, since they only aim to improve and evaluate the standard error, it is unclear if the robust error improves. We believe that since RST provides a complementary statistical perspective on the tradeoff, it can be combined with methods like IAT, NAS or AdvProp to see further gains in standard and robust errors. We leave this to future work.

\section{Conclusion}
We study the commonly observed increase in standard error upon adversarial training due to generalization from finite data in a well-specified setting with consistent perturbations. Surprisingly, we show that methods that augment the training data with consistent perturbations, such as adversarial training, can increase the standard error even in the simple setting of noiseless linear regression where the true linear function has zero standard and robust error. Our analysis reveals that the mismatch between the inductive bias of models and the underlying distribution of the inputs causes the standard error to increase even when the augmented data is perfectly labeled. This insight motivates a method that provably eliminates the tradeoff in linear regression by incorporating an appropriate regularizer that utilizes the distribution of the inputs. While not immediately apparent, we show that this is a special case of the recently proposed robust self-training (RST) procedure that uses additional unlabeled data to estimate the distribution of the inputs. Previous works view RST as a method to improve the robust error by increasing the sample size. Our work provides some theoretical justification for why RST improves \emph{both} the standard and robust error, thereby mitigating the tradeoff between accuracy and robustness. How to best utilize unlabeled data, and whether sufficient unlabeled data can completely eliminate the tradeoff remain open questions.

\newpage
\subsection*{Acknowledgements}

We are grateful to Tengyu Ma, Yair Carmon, Ananya Kumar, Pang Wei Koh, Fereshte Khani, Shiori Sagawa and Karan Goel for valuable discussions and comments. This work was funded by an Open Philanthropy Project Award and NSF Frontier Award as part of the Center for Trustworthy Machine Learning (CTML). AR was supported by Google Fellowship and Open Philanthropy AI Fellowship. SMX was supported by an NDSEG Fellowship. FY was supported by the Institute for Theoretical Studies ETH Zurich and the Dr. Max Rossler and the Walter Haefner Foundation. FY and JCD were supported by the Office of Naval Research Young Investigator Awards.

\bibliography{refdb/all}
\bibliographystyle{icml2020}

\appendix
\onecolumn
\section{Transformations to handle arbitrary matrix norms}
\label{sec:app-matrix-norms}
Consider a more general minimum norm estimator of the following form.
Given inputs $X$ and corresponding targets $y$ as training data, we study the interpolation estimator,
\begin{align}
  \label{eqn:min-norm}
  \hat{\theta} &= \arg \min \limits_{\theta} \Big \{ \theta^\top \minmat \theta :  X \theta = y \Big \}, 
\end{align}
where $\minmat$ is a positive definite (PD) matrix that incorporates prior knowledge about the true model.
For simplicity, we present our results in terms of the $\ell_2$ norm (ridgeless regression) as defined in Equation~\ref{eqn:min-norm}. However, all our results hold for arbitrary $\minmat$--norms via appropriate rotations. Given an arbitrary PD matrix $\minmat$, the rotated covariates $x \leftarrow \minmat^{-1/2} x$ and rotated parameters $\theta \leftarrow \minmat^{1/2} \theta$ maintain $y = X \theta$ and the $\minmat$-norm of parameters simplifies to $\|\theta\|_2$.

\section{Standard error of minimum norm interpolants}
\label{sec:app-bias}

\subsection{Projection operators}
\label{sec:proj-matrices}
The projection operators $\pistd$ and $\piaug$ are formally defined as follows.
\begin{align}
  \sigmastd = \xstd^\top\xstd,&~~\pistd = I - \sigmastd^+ \sigmastd \\
  \sigmaaug = \xstd^\top\xstd + \xext^\top\xext,&~~\piaug = I - \sigmaaug^+ \sigmaaug.
\end{align}

\subsection{Invariant transformations may have arbitrary nullspace components}
\label{sec:arbitrary-nullspace}

We show that the transformations which satisfy the invariance condition $(\tilde{x}-x)^\top\theta^\star=0$ where $\tilde{x}\in T(x)$ is a transformation of $x$ may have arbitrary nullspace components for general transfomation mappings $T$.
Let $\piparastd$ and $\pistd$ be the column space and nullspace projections for the original data $\xstd$.
The invariance condition is equivalent to
\begin{align}
    (\tilde{x}-x)^\top\theta^\star&=(\piparastd(\tilde{x} - x) + \pistd(\tilde{x}-x))^\top \theta^\star = 0
\end{align}
which implies that as long as $\pistd\theta^\star \neq 0$, then for any choice of nullspace component $\pistd(\tilde{x}) \in \Null(\xstd^\top \xstd)$, there is a choice of $\piparastd\tilde{x}$ which satisfies the condition.
Thus, we consider augmented points $\xext$ with arbitrary components in the nullspace of $\xstd$.

\subsection{Proof of Theorem~\ref{thm:main}}
\label{sec:app-existltwo}

Inequality~\eqref{eqn:exactwv} follows from
\begin{align}
    \stderr(\augest) - \stderr(\stdest) &= (\thetatrue - \augest )^\top \sigmapop (\thetatrue-\augest) - (\thetatrue - \stdest)^\top \sigmapop (\thetatrue-\stdest) \nonumber \\
    &= (\piaug \thetatrue )^\top \sigmapop \piaug \thetatrue - (\pistd \thetatrue)^\top \sigmapop \pistd \thetatrue \nonumber \\
    &= w^\top \sigmapop w - (w+v)^\top \sigmapop (w+v) \nonumber\\
    &= -2w^\top \sigmapop v -   v^\top \sigmapop v \label{eq:exactre}
\end{align}
by decomposition of $\pistd \thetatrue = v+ w$ where $v = \pistd \piparaaug \thetatrue$ and $w = \pistd \piaug \thetatrue$.
Note that the error difference does scale with $\|\thetatrue\|^2$, although the sign of the difference does not.

\subsection{Proof of Corollary~\ref{thm:cor}}
\label{sec:app-corollary1}
Corollary~\ref{thm:cor} presents three sufficient conditions under which the standard error of the augmented estimator $\stderr(\augest)$ is never larger than the standard error of the standard estimator $\stderr(\stdest)$.
\begin{enumerate}
\item When the population covariance $\sigmapop = I$, from Theorem~\ref{thm:main}, we see that
  \begin{align}
    \stderr(\stdest)  - \stderr(\augest) = v^\top v + 2 w^\top v = v^\top v \geq 0, 
  \end{align}
  since $v = \pistd \piparaaug \thetatrue$ and $w = \piaug \thetatrue$ are orthogonal.
\item When $\piaug = 0$, the vector $w$ in Theorem~\ref{thm:main} is $0$, and hence we get
  \begin{align}
    \stderr(\stdest)  - \stderr(\augest) = v^\top v \geq 0. 
  \end{align}
\item We prove the eigenvector condition in Section~\ref{sec:characterization} which studies the effect of augmenting with a single extra point in general. 
\end{enumerate}

\subsection{Proof of Proposition~\ref{prop:simple-complex}}
\label{sec:app-lowerbound}

The proof of Proposition~\ref{prop:simple-complex} is based on the following two lemmas that are also useful for characterization purposes in Corollary~\ref{cor:characterization}.

\begin{lemma}
\label{lem:simplela}
If a PSD matrix $\sigmapop$ has non-equal eigenvalues, one can find two unit vectors $w,v$ for which the following holds
\begin{align}
\label{eq:simplela}
    w^\top v=0 \qquad \text{and} \qquad w^\top \sigmapop v \neq 0
\end{align}
Hence, there exists a combination of original and augmentation dataset $\xstd, \xext$ such that condition~\eqref{eq:simplela} holds for two directions $v \in \Col(\pistd \piparaaug)$ and $w \in  \Col(\pistd \piaug)=\Col(\piaug)$.
\end{lemma}
Note that neither $w$ nor $v$ can be eigenvectors of $\sigmapop$ in order for both conditions in equation~\eqref{eq:simplela} to hold. Given a population covariance, fixed original and augmentation data for which condition~\eqref{eq:simplela} holds, we can now explicitly construct $\thetatrue$ for which augmentation increases standard error.

\begin{lemma}
\label{lem:constructexist}
Assume $\sigmapop, \xstd, \xext$ are fixed. Then condition~\eqref{eq:simplela} holds for two directions $v \in \Col(\pistd \piparaaug)$ and $w \in  \Col(\pistd \piaug)$ iff there exists a $\thetatrue$ such that $\stderr(\augest) - \stderr(\stdest) \geq c$ for some $c>0$.
Furthermore, the $\ell_2$ norm of $\thetatrue$ needs to satisfy the following lower bounds with $c_1 := \|\augest\|^2 - \|\stdest\|^2$
\begin{align}
    \|\thetatrue\|^2 - \|\augest\|^2 &\geq \beta_1 c_1 +\beta_2 \frac{c^2}{c_1} \nonumber\\
    \|\thetatrue\|^2 - \|\stdest\|^2 &\geq (\beta_1 + 1)c_1 + \beta_2 \frac{c^2}{c_1} \label{eq:normdiff}
\end{align}
where $\beta_i$ are constants that depend on $\xstd, \xext, \sigmapop$.
\end{lemma}

Proposition~\ref{prop:simple-complex} follows directly from the second statement of Lemma~\ref{lem:constructexist}
by minimizing the bound~\eqref{eq:normdiff} with respect to $c_1$ which is a free parameter to be
chosen during construction of $\thetatrue$ (see proof of Lemma~\eqref{lem:constructexist}.
The minimum is attained for $c_1 = 2 \sqrt{(\beta_1 +1) (\beta_2 c^2)}$.
We hence conclude that $\thetatrue$ needs to be sufficiently more complex than a good standard solution, i.e. $\|\thetatrue\|^2_2 - \| \stdest\|^2_2 > \gamma c$ where $\gamma > 0$ is a constant that depends on the $\xstd, \xext$.

\subsection{Proof of technical lemmas}
In this section we prove the technical lemmas that are used to prove Theorem~\ref{thm:main}.
\subsubsection{Proof of Lemma~\ref{lem:constructexist}}

Any vector $\pistd\theta\in \Null(\sigmastd)$ can be decomposed into orthogonal components $\pistd\theta = \pistd\piaug\theta + \pistd \piparaaug \theta$.
Using the minimum-norm property, we can then always decompose the (rotated) augmented estimator $\augest \in \Col (\piaug)=\Col(\pistd\piaug)$ and true parameter $\thetatrue$ by
\begin{align*}
    \augest &= \stdest + \sum_{v_i \in \ext} \zeta_i v_i\\
    \thetatrue &= \augest + \sum_{w_j\in \rest}  \xi_j w_j,
\end{align*}
where we define ``$\ext$'' as the set of basis vectors which span
$\Col(\pistd \piparaaug)$ and respectively ``$\rest$''
for $\Null(\sigmaaug)$.  Requiring the standard error increase to be some
constant $c>0$ can be rewritten using identity~\eqref{eq:exactre} as
follows
\begin{align}
    \stderr(\augest)- \stderr(\stdest) &= c  \nonumber\\
        \iff (\sum_{v_i \in \ext} \zeta_i v_i)^\top \sigmapop (\sum_{v_i \in \ext} \zeta_i v_i) + c &= - 2(\sum_{w_j \in \rest}  \xi_j w_j) \sigmapop (\sum_{v_i \in \ext} \zeta_i v_i) \nonumber\\
        \iff (\sum_{v_i \in \ext} \zeta_i v_i)^\top \sigmapop (\sum_{v_i \in \ext} \zeta_i v_i) + c&= -2 \sum_{w_j \in \rest, v_i \in \ext} \xi_j \zeta_i w_j^\top \sigmapop v_i  \label{eq:exacttheta}
\end{align}
The left hand side of equation~\eqref{eq:exacttheta} is always
positive, hence it is necessary for this equality to hold with any
$c>0$, that there exists at least one pair $i,j$ such that $w_j^\top
\sigmapop v_i \neq 0$ and one direction of the iff statement is
proved.

For the other direction, we show that if there exist $v
\in \Col(\pistd \piparaaug)$ and $w
\in \Col(\pistd \piaug)$ for which
condition~\eqref{eq:simplela} holds (wlog we assume that the $w^\top
\sigmapop v < 0$) we can construct a $\thetatrue$
for which the inequality~\eqref{eqn:exactwv} in Theorem~\ref{thm:main} holds as follows:

It is then necessary by our assumption that $\xi_j \zeta_i w_j^\top \sigmapop v_i > 0$ for at least some $i, j$. We can then set $\zeta_i > 0$ such that  $\|\augest-\stdest\|^2 = \|\zeta\|^2 = c_1 >0$, i.e. that the augmented estimator is not equal to the standard estimator (else obviously there can be no difference in error and equality~\eqref{eq:exacttheta} cannot be satisfied for any desired error increase $c>0$).

The choice of $\xi$ minimizing $\|\theta^\star - \augest\|^2 = \sum_j \xi_j^2$ that also satisfies equation~\eqref{eq:exacttheta} is an appropriately scaled vector in the direction of $x = W^\top \sigmapop V \zeta$ where we define $W :=[w_1, \dots, w_{|\rest|}]$ and $V := [v_1, \dots, v_{|\ext|}]$. Defining $c_0 = \zeta^\top V^\top \sigmapop V \zeta$ for convenience and then setting
\begin{equation}
\label{eq:xidef}
    \xi = - \frac{c_0 + c}{2\|x\|^2_2} x
\end{equation}
which is well-defined since $x\neq 0$, yields a $\thetatrue$ such that augmentation increases standard error. It is thus necessary for $\stderr(\augest) - \stderr(\stdest) = c$ that
\begin{align*}
    \sum_j \xi_j^2 &= \frac{(c_0 + c)^2}{4 \| W^\top \sigmapop V\zeta\|^2} = \frac{(\zeta^\top V^\top\sigmapop V \zeta + c)^2}{4\zeta^\top V^\top \sigmapop W W^\top \sigmapop V \zeta}\\
    &\geq \frac{(\zeta^\top V^\top\sigmapop V \zeta)^2}{4 \zeta^\top V^\top \sigmapop W W^\top \sigmapop V \zeta} + \frac{c^2}{4 \zeta^\top V^\top \sigmapop W W^\top \sigmapop V \zeta}\\
    &\geq \frac{c_1}{4} \frac{\lambda_{\min}^2(V^\top \sigmapop V)}{\lambda^2_{\max}( W^\top \sigmapop V)} + \frac{c^2}{4 c_1 \lambda^2_{\max}(W^\top \sigmapop V)}.
\end{align*}
By assuming existence of $i,j$ such that $\xi_j \zeta_i w_j^\top \sigmapop v_i \neq 0$, we are guaranteed that $\lambda^2_{\max}(W^\top \sigmapop V) > 0$.

Note due to construction we have $\|\thetatrue\|_2^2= \|\stdest\|_2^2 + \sum_i \zeta_i^2 + \sum_j\xi_j^2$ and plugging in the choice of $\xi_j$ in equation~\eqref{eq:xidef} we have
\begin{align*}
    \|\thetatrue\|_2^2 - \|\stdest\|_2^2 &\geq c_1 \left[ 1 + \frac{\lambda_{\min}^2(V^\top \sigmapop V)}{4 \lambda^2_{\max}( W^\top \sigmapop V)}\right] + \frac{c^2}{4  \lambda^2_{\max}(W^\top \sigmapop V)} \frac{1}{c_1}.\\
\end{align*}
Setting $\beta_1 = \left[ 1 + \frac{\lambda_{\min}^2(V^\top \sigmapop V)}{4 \lambda^2_{\max}( W^\top \sigmapop V)}\right] $, $\beta_2 = \frac{1}{4 \lambda^2_{\max}(W^\top \sigmapop V)}$ yields the result.

\subsubsection{Proof of Lemma~\ref{lem:simplela}}
Let $\lambda_1,\dots, \lambda_m$ be the $m$ non-zero eigenvalues of $\sigmapop$ and $u_i$ be the corresponding eigenvectors.
Then choose $v$ to be any combination of the eigenvectors $v = U \beta$ where $U = [u_1, \dots, u_m]$ where at least $\beta_i, \beta_j \neq 0$ for $\lambda_i \neq \lambda_j$.
We next construct $w = U \alpha$ by choosing $\alpha$ as follows such that the inequality in~\eqref{eq:simplela} holds:
\begin{align*}
    \alpha_i = \frac{\beta_j}{\beta_i^2 + \beta_j^2}\\
    \alpha_j = \frac{- \beta_i}{\beta_i^2 + \beta_j^2}
\end{align*}
and $\alpha_k=0$ for $k\neq i,j$. Then we have that $\alpha^\top \beta = 0$ and hence $w^\top v =0$. Simultaneously
\begin{align*}
    w^\top \sigmapop v &= \lambda_i \beta_i \alpha_i + \lambda_j \beta_j \alpha_j \\
    &= (\lambda_i -\lambda_j) \frac{\beta_i \beta_j}{\beta_i^2 + \beta_j^2} \neq 0
\end{align*}
which concludes the proof of the first statement.

We now prove the second statement by constructing $\sigmastd=\xstd^\top\xstd, \sigmaext=\xext^\top\xext$ using $w,v$. We can then obtain $\xstd, \xext$ using any standard decomposition method to obtain $\xstd, \xext$.
We construct $\sigmastd, \sigmaext$ using $w,v$.
Without loss of generality, we can make them simultaneously diagonalizable. We construct a set of eigenvectors that is the same for both matrices paired with different eigenvalues. Let the shared eigenvectors include $w,v$. Then if we set the corresponding eigenvalues $\lambda_w(\sigmaext) = 0, \lambda_v(\sigmaext) > 0$ and $\lambda_w(\sigmastd) = 0, \lambda_v(\sigmastd) = 0$, then $\lambda_w(\sigmaaug)= 0$ such that $w\in \Col(\pistd \piaug)$ and $v\in \Col(\pistd \piparaaug)$. This shows the second statement. With this, we can design a $\thetatrue$ for which augmentation increases standard error as in Lemma~\ref{lem:constructexist}.


\subsection{Characterization Corollary~\ref{cor:characterization}}
\label{sec:characterization}

A simpler case to analyze is when we only augment with one extra data point. The following corollary characterizes which single augmentation directions lead to higher prediction error for the augmented estimator.
\begin{restatable}[]{corollary}{characterizationsingle}
\label{cor:characterization}
The following characterizations hold for augmentation directions that do not cause the standard error of the augmented estimator to be higher than the original estimator.
\begin{enumerate}[(a)]
    \item \emph{(in terms of ratios of inner products)} For a given $\thetatrue$, data augmentation does not increase the standard error of the augmented estimator for a single augmentation direction $\singleaug$ if
    \begin{equation}
        \label{eq:singleaug}
                \frac{\singleaug^\top \pistd \sigmapop \pistd \singleaug}{\singleaug^\top \pistd \singleaug} -  2 \frac{(\pistd \singleaug)^\top \sigmapop \pistd \theta^\star}{\singleaug^\top \pistd \theta^\star} \leq 0
            \end{equation}
    \item \emph{(in terms of eigenvectors)} Data augmentation does not increase standard error for any $\thetatrue$ if $\pistd \singleaug$ is an eigenvector of $\sigmapop$. However if one augments in the direction of a mixture of eigenvectors of $\sigmapop$ with different eigenvalues, there exists $\thetatrue$ such that augmentation increases standard error.

    \item \emph{(depending on well-conditioning of $\sigmapop$)} If $\frac{\lambda_{\max}(\sigmapop)}{\lambda_{\min}(\sigmapop)} \leq 2$ and $\pistd \thetatrue$ is an eigenvector of $\sigmapop$, then no augmentations $\singleaug$ increase standard error.
\end{enumerate}
\end{restatable}

The form in Equation~\eqref{eq:singleaug} compares ratios of inner products of $\pistd \singleaug$ and $\pistd \theta^\star$ in two spaces: the one in the  numerator is weighted by $\sigmapop$ whereas the denominator is the standard inner product.
Thus, if $\sigmapop$ scales and rotates rather inhomogeneously, then augmenting with $\singleaug$ may hurt standard error.
Here again, if $\sigmapop=\gamma I$ for $\gamma>0$, then the condition must hold.

\subsubsection{Proof of Corollary~\ref{cor:characterization} (a)}
Note that for a single augmentation point $\xext = \singleaug^\top$, the orthogonal decomposition of $\pistd \thetatrue$ into $\Col(\piaug)$ and $\Col(\pistd\piparaaug)$ is defined by $v = \frac{{\pistd\singleaug}^\top \thetatrue}{\|{\pistd\singleaug}\|^2} {\pistd\singleaug}$ and $w= \pistd \thetatrue - v$ respectively. Plugging back into into identity~\eqref{eq:exactre} then
yields the following condition for safe augmentations:
\begin{align}
    &2(v-\pistd \thetatrue)^\top \sigmapop v- v^\top\sigmapop v \leq 0 \label{eq:pred_diff}\\
    &v^\top \sigmapop v - 2 (\pistd \thetatrue)^\top \sigmapop v \leq 0 \nonumber\\
    \iff &{\pistd\singleaug}^\top \sigmapop {\pistd\singleaug} \leq 2 (\pistd \thetatrue)^\top \sigmapop {\pistd\singleaug} \cdot \frac{\|{\pistd\singleaug}\|^2}{{\pistd\singleaug}^\top \thetatrue} \nonumber
\end{align}
Rearranging the terms yields inequality \eqref{eq:singleaug}.

Safe augmentation directions for specific choices of $\thetatrue$ and $\sigmapop$ are illustrated in Figure~\ref{fig:simpleexistence}.

\subsubsection{Proof of Corollary~\ref{cor:characterization} (b)}
\label{app:minimax}

Assume that $\pistd \singleaug$ is an eigevector of $\sigmapop$ with eigenvalue $\lambda > 0$.
We have
\begin{equation*}
    \frac{\singleaug^\top \pistd \sigmapop \pistd \singleaug}{\singleaug^\top \pistd \singleaug} -  2 \frac{(\pistd \singleaug)^\top \sigmapop \pistd \theta^\star}{\singleaug^\top \pistd \theta^\star} = -\lambda < 0
\end{equation*}
for any $\thetatrue$.
Hence by Corollary~\ref{cor:characterization} (a), the standard error doesn't increase by augmenting with eigenvectors of $\sigmapop$ for any $\thetatrue$.

When the single augmentation direction $v$ is not an eigenvector of $\sigmapop$, by Lemma~\ref{lem:simplela} one can find $w$ such that $w^\top \sigmapop v \neq 0$. The proof in Lemma~\ref{lem:simplela} gives an explicit construction  for $w$ such that condition~\eqref{eq:simplela} holds and the result then follows directly by Lemma~\ref{lem:constructexist}.

\subsubsection{Proof of Corollary~\ref{cor:characterization} (c)}
\label{app:well-conditioning}
Suppose $\sigmapop \pistd \theta^\star = \lambda \pistd \theta^\star $ for some $\lambda_{\min}(\sigmapop) \leq \lambda \leq \lambda_{\max}(\sigmapop)$.
Then starting with the expression \eqref{eq:singleaug},
\begin{align*}
    \frac{\singleaug^\top \pistd \sigmapop \pistd \singleaug}{\singleaug^\top \pistd \singleaug} -  2 \frac{(\pistd \singleaug)^\top \sigmapop \pistd \theta^\star}{\singleaug^\top \pistd \theta^\star} &=
    \frac{\singleaug^\top \pistd \sigmapop \pistd \singleaug}{\singleaug^\top \pistd \singleaug} -  2\lambda\\
    &\leq \lambda_{\max}(\sigmapop) -  2\lambda < 0
\end{align*}
by applying $\frac{\lambda_{\max}(\sigmapop)}{\lambda_{\min}(\sigmapop)}\leq 2$. Thus when $\pistd \theta^\star$ is an eigenvector of $\sigmapop$, there are no augmentations $\singleaug$ that increase the standard error.

\section{Details for spline staircase}
\label{app:splines}

We describe the data distribution, augmentations, and model details for the spline experiment in Figure~\ref{fig:sample-size} and toy scenario in Figure~\ref{fig:spline}. Finally, we show that we can construct a simplified family of spline problems where the ratio between standard errors of the augmented and standard estimators increases unboundedly as the number of stairs.

\subsection{True model}
\label{sec:splinedist}

We consider a finite input domain  
\begin{equation}
\label{eq:defT}
    \sT = \{0, \epsilon, 1, 1+\epsilon, \hdots, s-1, s-1+\epsilon\}
\end{equation}
for some integer $s$ corresponding to the total number of ``stairs'' in the staircase problem.
Let $\tline \subset \sT = \{0,1,\dots,s-1\}$.
We define the underlying function $f^\star: \R \mapsto \R$ as $f^\star(t) = \lfloor t \rfloor$. This function takes a staircase shape, and is linear when restricted to $\tline$.

\paragraph{Sampling training data $\xstd$} We describe the data distribution in terms of the one-dimensional input $t$, and by the one-to-one correspondence with spline basis features $x=X(t)$, this also defines the distribution of spline features $x\in\sX$. Let $w\in \Delta_s$ define a distribution over $\tline$ where $\Delta_s$ is the probability simplex of dimension $s$.
We define the data distribution with the following generative process for one sample $t$.
First, sample a point $i$ from $\tline$ according to the categorical distribution described by $w$, such that $i\sim \text{Categorical}(w)$.
Second, sample $t$ by perturbing $i$ with probability $\delta$ such that
\[
t =
\begin{cases}
    i & \text{w.p. } 1-\delta\\
    i+\epsilon & \text{w.p. } \delta.
\end{cases}
\]
The sampled $t$ is in $\tline$ with probability $1-\delta$ and $\tline^c$ with probability $\delta$, where we choose $\delta$ to be small.

\paragraph{Sampling augmented points $\xext$}
For each element $t_i$ in the training set, we augment with $\tilde{T}_i=[u \uarsim ~B(t_i)]$, an input chosen uniformly at random from $B(t_i)=\{\lfloor t_i \rfloor, \lfloor t_i \rfloor + \epsilon\}$. 
Recall that in our work, we consider data augmentation where the targets associated with the augmented points are from the ground truth oracle.
Notice that by definition, $f^\star(\tilde{t}_i) = f^\star(t_i)$ for all $\tilde{t}\in B(t_i)$, and thus we can set the augmented targets to be $\tilde{y}_i = y_i$.
This is similar to random data augmentation in images~\citep{yaeger1996effective, krizhevsky2012imagenet}, where inputs are perturbed in a way that preserves the label.

\subsection{Spline model}
We parameterize the spline predictors as $f_{\theta}(t) = \theta^\top X(t)$ where $X: \R \rightarrow \R^d$ is the cubic B-spline feature mapping~\citep{friedman2001elements} and the norm of $f_\theta(t)$ can be expressed as $\theta^\top \minmat \theta$ for a matrix $\minmat$ that penalizes a large second derivative norm where $[\minmat]_{ij} = \int X_i^{''} (u) X_j^{''} (u)du$.
Notice that the splines problem is a linear regression problem from $\R^d$ to $\R$ in the feature domain $X(t)$, allowing direct application of Theorem~\ref{thm:main}.
As a linear regression problem, we define the finite domain as $\sX=\{X(t): t\in\sT\}$ containing $2s$ elements in $\R^d$.
There is a one-to-one correspondence between $t$ and $X(t)$, such that $X^{-1}$ is well-defined.
We define the features that correspond to inputs in $\tline$ as $\xline = \{x: X^{-1}(x) \in \tline\}$.
Using this feature mapping, there exists a $\theta^\star$ such that $f_{\theta^\star}(t) = f^\star(t)$ for $t\in \sT$.

Our hypothesis class is the family of cubic B-splines as defined in~\citep{friedman2001elements}.
Cubic B-splines are piecewise cubic functions, where the endpoints of each cubic function are called the knots.
In our example, we fix the knots to be
$[0, \epsilon, 1, \dots, s-1, s-1+\epsilon]$,
which places a knot on every point in $\sT$. This ensures that the function class contains an interpolating function on all $t \in \sT$, i.e. for some $\theta^\star$,
\begin{align*}
    f_{\theta^\star}(t) = {\theta^\star}^\top  X(t) = f^\star(t)=\lfloor t \rfloor.
\end{align*}

We solve the minimum norm problem
\begin{align}
    \stdest = \argmin_{\theta} \{\theta^\top \minmat \theta: \xstd\theta = \ystd\}
\end{align}
for the standard estimator and the corresponding augmented problem to obtain the augmented estimator.

\begin{figure}[t]
  \centering
  \begin{subfigure}{0.3\textwidth}
    \centering
    \includegraphics[scale=0.37]{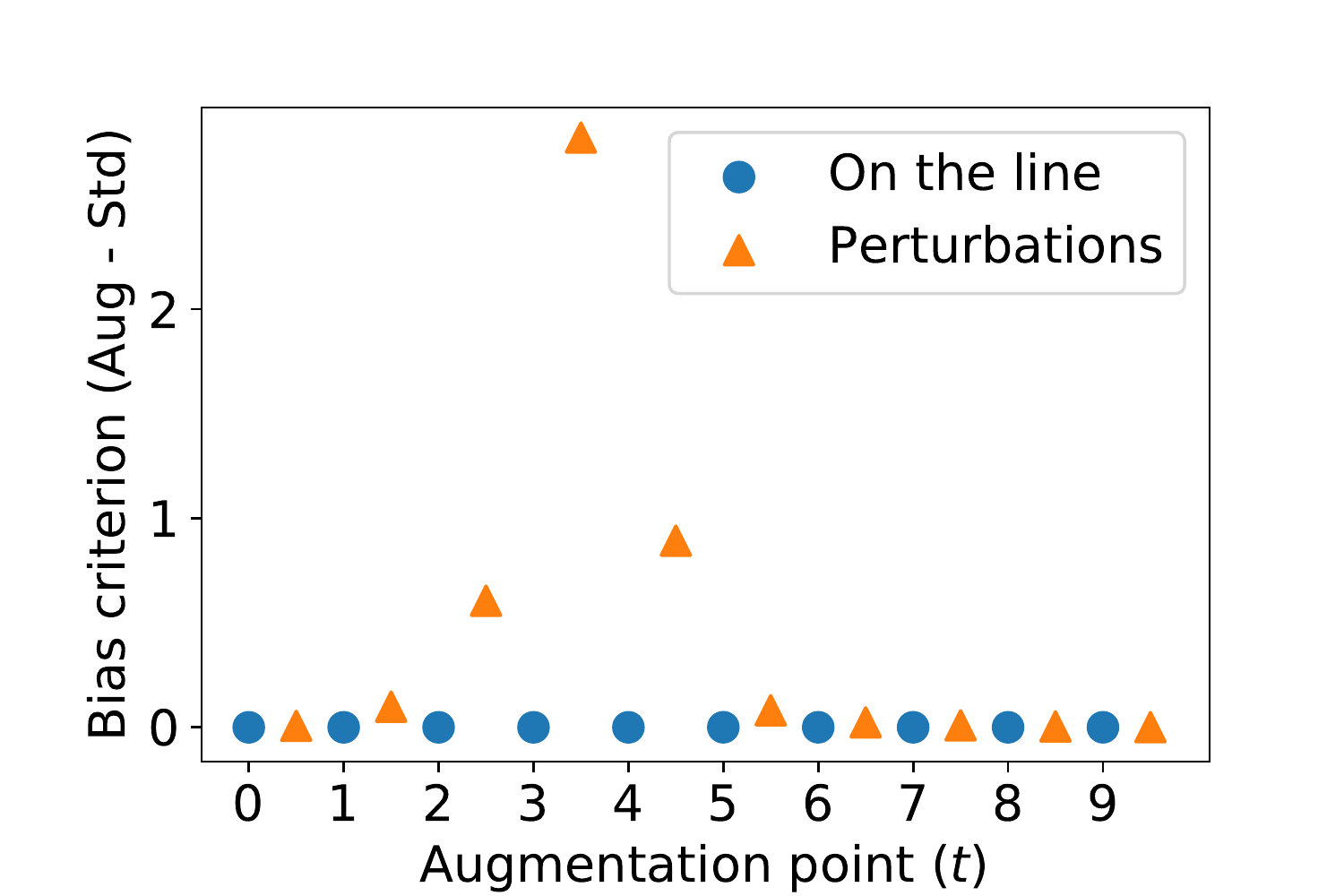}
  \end{subfigure}
  \begin{subfigure}{0.3\textwidth}
      \centering
      \includegraphics[scale=0.3]{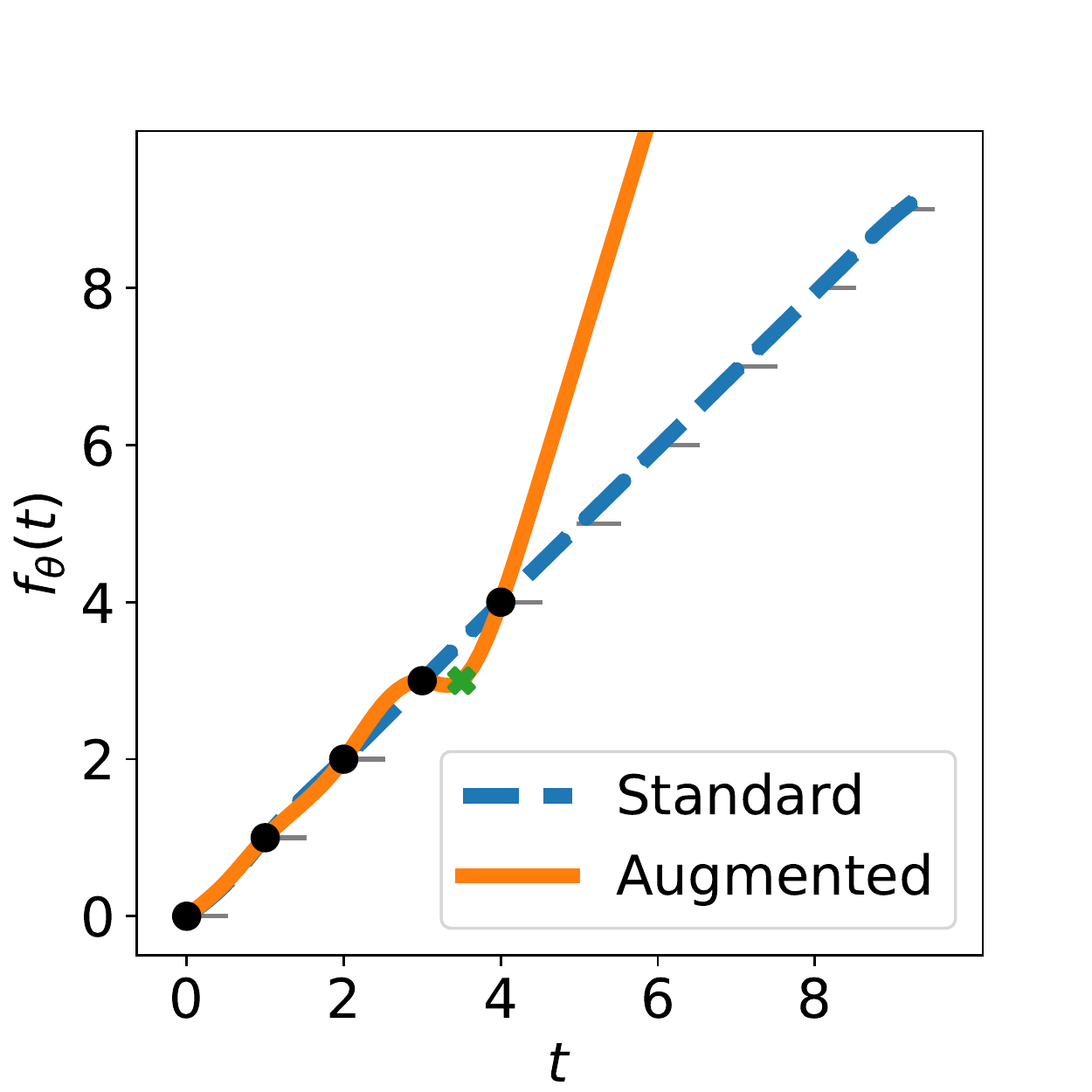}
      \caption{Augment with $x=\Phi(3.5)$}
  \end{subfigure}
  \begin{subfigure}{0.3\textwidth}
      \centering
      \includegraphics[scale=0.3]{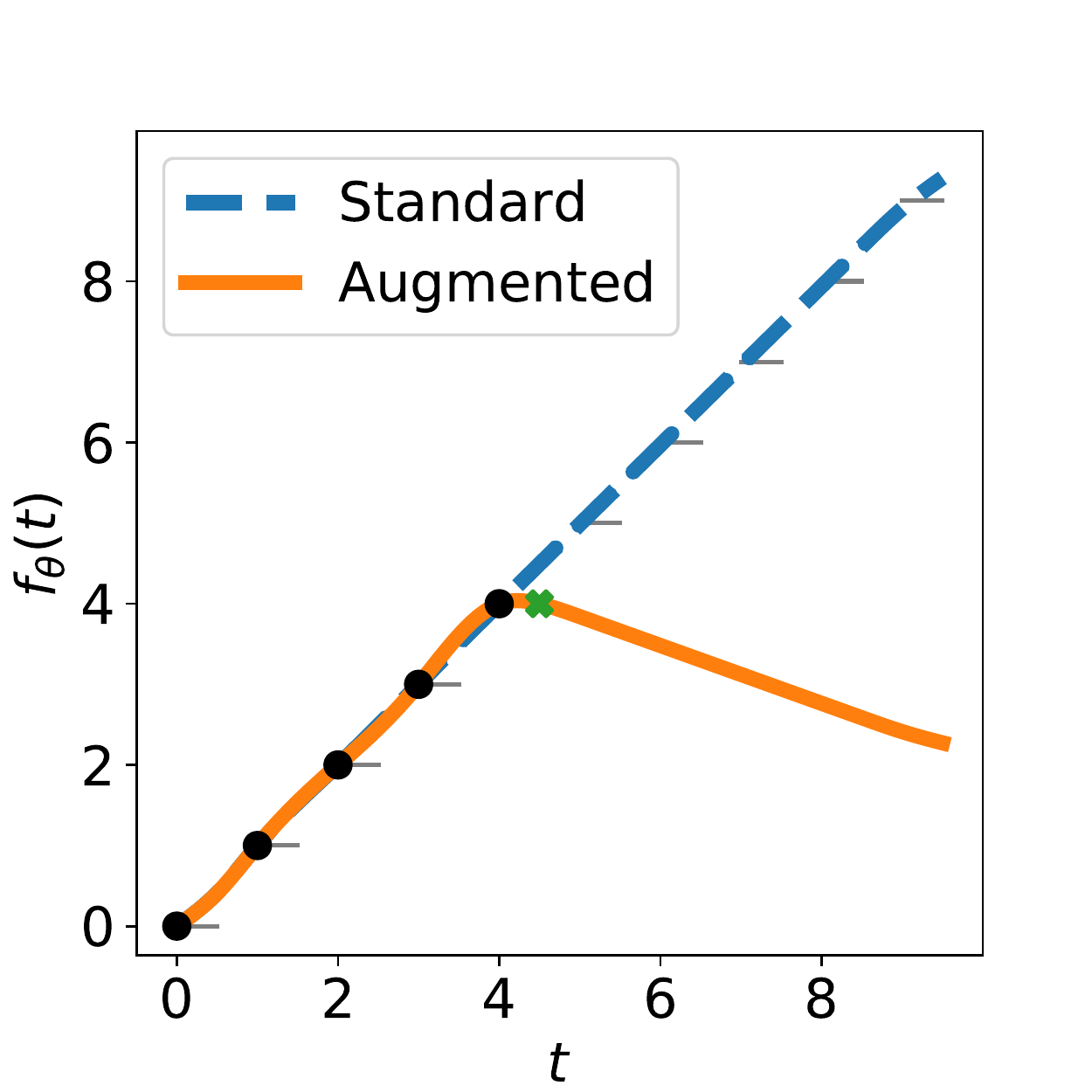}
      \caption{Augment with $x=\Phi(4.5)$}
  \end{subfigure}
    \caption{
    Visualization of the effect of single augmentation points in the noiseless spline problem given an initial dataset $\xstd = \{\Phi(t): t\in\{0, 1, 2, 3, 4\}\}$. The standard estimator defined by $\xstd$ is linear.
      \textbf{(a)} Plot of the difference term in Corollary~\ref{cor:characterization} (a), is positive when augmenting a single point causes higher test error. Augmenting with points on $\xline$ does not affect the bias, but augmenting with any element of $\{X(t): t\in \{2.5, 3.5, 4.5\}\}$ hurts the bias of the augmented estimator dramatically.
      \textbf{(b), (c)} Augmenting with $X(3.5)$ or $X(4.5)$ hurts the bias by changing the direction of extrapolation. 
      }
    \label{fig:spline_condition}
\end{figure}

\subsection{Evaluating Corollary~\ref{cor:characterization} (a) for splines}
We now illustrate the characterization for the effect of augmentation with different single points in 
Theorem~\ref{cor:characterization} (a) on the splines problem.
We assume the domain to $\sT$ as defined in equation~\ref{eq:defT} with $s=10$ and our training data
to be $\xstd=\{ X(t): t\in\{0,1,2,3,4\}\}$. Let $\emph{local}$ perturbations be spline features for $\tilde{t}\notin \tline$ where $\tilde{t}=t + \epsilon$ is $\epsilon$ away from some $t\in\{0,1,2,3,4\}$ from the training set.
We examine all possible single augmentation points in Figure~\ref{fig:spline_condition} (a) and plot the calculated standard error difference as defined in equation~\eqref{eq:pred_diff}.
Figure~\ref{fig:spline_condition} shows that augmenting with an additional point from $\{ X(t): t\in \tline\}$ does not affect the bias, but adding any perturbation point in $\{ X(\tilde{t}): \tilde{t}\in \{2.5, 3.5, 4.5\}\}$ where $\tilde{t}\notin \tline$ increases the error significantly by changing the direction in which the estimator extrapolates.
Particularly, \emph{local} augmentations near the boundary of the original dataset hurt the most while other augmentations do not significantly affect the bias of the augmented estimator.

\subsubsection{Local and global structure in the spline staircase}
\label{sec:splines-local-global}
\begin{figure}[t]
    \centering
      \includegraphics[scale=0.37]{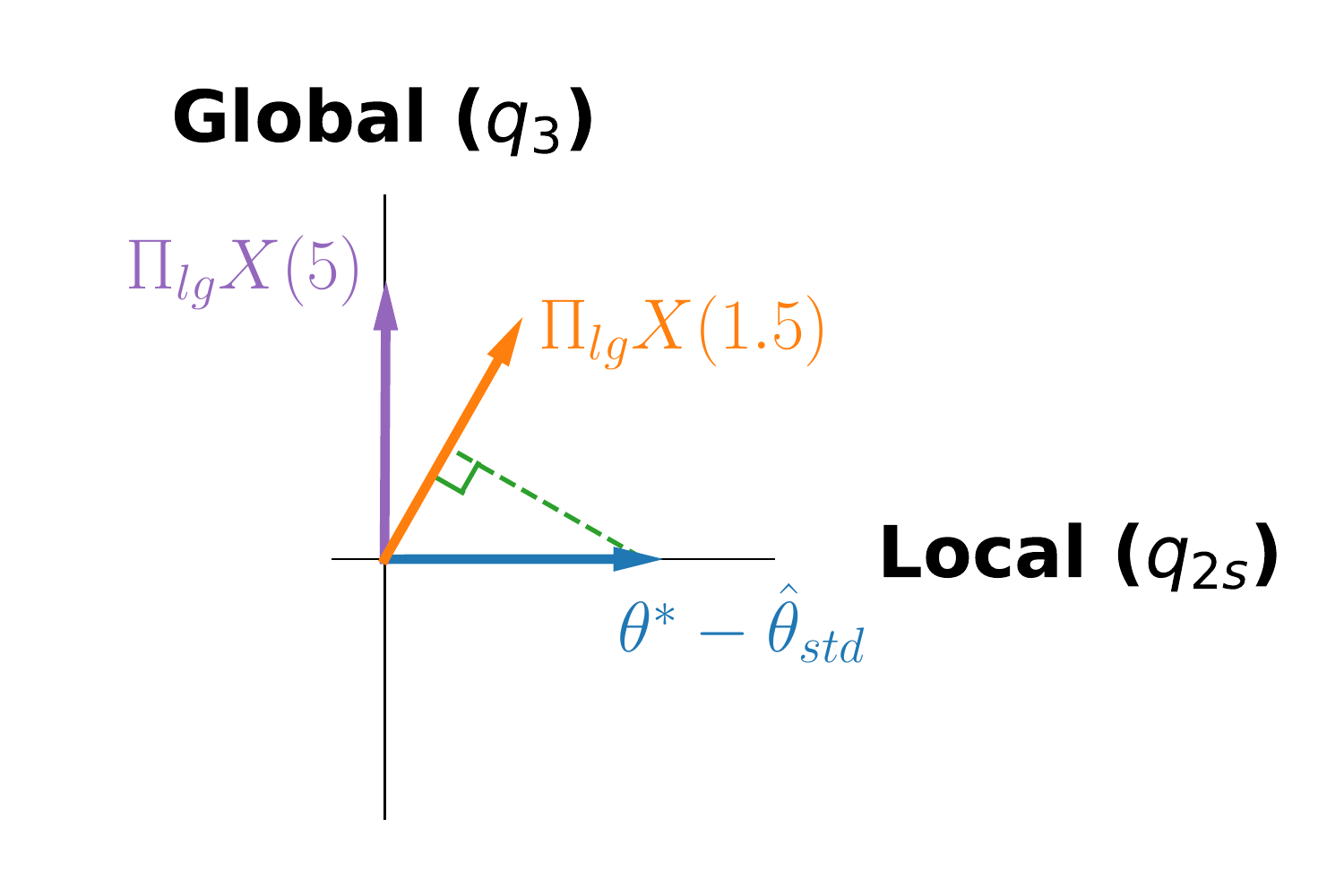}
    \caption{
      Nullspace projections onto global direction $q_3$ and local direction $q_{2s}$ in $\Null(\sigmapop)$ via $\pilocalglobal$, representing global and local eigenvectors respectively. The local perturbation $\pilocalglobal\hat{\Phi}(1.5)$ has both local and global components, creating a high-error component in the global direction.
      }
    \label{fig:K_eigenvectors_2}
\end{figure}

In the spline staircase, the local perturbations can be thought of as fitting high frequency noise in the function space, where fitting them causes a global change in the function.

To see this, we transform the problem to minimum $\ell_2$ norm linear interpolation using features $X_M(t) = X(t) M^{-1/2}$
so that the results from Section~\ref{sec:general} apply directly. Let $\sigmapop$ be the population covariance of $X_M$ for a uniform distribution over the discrete domain consisting of $s$ stairs and their perturbations (Figure~\ref{fig:spline}). Let $Q=[q_i]_{i=1}^{2s}$
be the eigenvectors of $\sigmapop$ in decreasing order of their corresponding eigenvalues. The visualization in Figure~\ref{fig:K_eigenvectors} shows that $q_i$ are wave functions in the original input space; the ``frequency'' of the wave increases as $i$ increases.

Suppose the original training set consists of two points, $\xstd = [X_M(0), X_M(1)]^\top$. We study the effect of augmenting point $\singleaug$ in terms of $q_i$ above. First, we find that the first two eigenvectors corresponding to linear functions satisfy $\pistd q_1 = \pistd q_2 = 0$. Intuitively, this is because the standard estimator is linear. For ease of visualization, we consider the 2D space in $\Null(\sigmapop)$ spanned by $\pistd q_3$ (global direction, low frequency) and $\pistd q_{2s}$ (local direction, high frequency). The matrix $\pilocalglobal=[\pistd q_3,~\pistd q_{2s}]^\top$ projects onto this space. Note that the same results hold when projecting onto all $\pistd q_i$ in $\Null(\sigmapop)$.

In terms of the simple 3-D example in Section~\ref{sec:simple-3D}, the global direction corresponds to the costly direction with large eigenvalue, as changes in global structure heavily affect the standard error. Figure~\ref{fig:K_eigenvectors_2} plots the projections $\pilocalglobal \thetatrue$ and $\pilocalglobal \xext$ for different $\xext$. 
When $\thetatrue$ has high frequency variations and is complex, $\pilocalglobal \thetatrue=(\thetatrue-\stdest)$ is aligned with the local dimension. For $\singleaug$ immediately local to training points, the projection $\pilocalglobal \singleaug$ (orange vector in Figure~\ref{fig:K_eigenvectors_2}) has both local and global components. 
Augmenting these local perturbations introduces error in the global component. For other $\singleaug$ farther from training points, $\pilocalglobal \singleaug$ (blue vector in Figure~\ref{fig:K_eigenvectors_2}) is almost entirely global and perpendicular to $\thetatrue-\stdest$, leaving bias unchanged. 
Thus, augmenting data close to original data cause estimators to fit local components at the cost of the costly global component which changes overall structure of the predictor like in Figure~\ref{fig:spline}(middle).
The choice of inductive bias in the $\minmat$--norm being minimized results in eigenvectors of $\sigmapop$ that correspond to local and global components, dictating this tradeoff.

\subsection{Data augmentation can be quite painful for splines}
\label{app:spline-err}
We construct a family of spline problems such that as the number the augmented estimator has much higher error than the standard estimator. We assume that our predictors are from the full family of cubic splines. 

\paragraph{Sampling distribution.}
We define a modified domain with continuous intervals $\sT = \cup_{t=0}^{s-1}[t, t+\epsilon]$.
Considering only $s$ which is a multiple of 2, we sample the original data set as described in
Section~\ref{sec:splinedist} with the following probability mass $w$:
\begin{align}
w(t) = \begin{cases}
\frac{1-\gamma}{s/2}  & t < s/2, \: t\in \tline\\
\frac{\gamma}{s/2}  & t \geq s/2, \: t\in \tline.
\end{cases}
\end{align}
for $\gamma\in [0,1)$. We define a probability distribution $P_{\sT}$ on $\sT$ for a random variable $T$ by setting $T = Z + S(Z)$ where  $Z\sim \text{Categorical}(w)$ and the $Z$-dependent perturbation $S(z)$ is defined as
\begin{align}
    S(z) \sim  \begin{cases}
  \text{Uniform}([z, z +\epsilon]) & \: \text{ w.p. } \delta\\
  z, & \: \text{ w.p. } 1-\delta.
\end{cases}
\end{align}
We obtain the training dataset $\xstd = \{ X(t_1), \dots,  X(t_n)\}$ by sampling $t_i \sim P_\sT$.

\paragraph{Augmenting with an interval.} 
Consider a modified augmented estimator for the splines problem, where for each point $t_i$ we augment with the entire interval $[\lfloor t_i \rfloor, \lfloor t_i \rfloor + \epsilon]$ with $\epsilon\in[0, 1/2)$ and the estimator is enforced to output $f_{\hat{\theta}}(x) = y_i = \lfloor t_i \rfloor$ for all $x$ in the  interval $[\lfloor t_i \rfloor, \lfloor t_i \rfloor + \epsilon]$.
Additionally, suppose that the ratio $s/n=O(1)$ between the number of stairs $s$ and the number of samples $n$ is constant. 

In this simplified setting, we can show that the standard error of the augmented estimator grows while the standard error of the standard estimator decays to 0.
\begin{theorem}
    \label{thm:spline-err}

Let the setting be defined as above.
Then with the choice of $\delta=\frac{\log(s^7) - \log(s^7-1)}{s}$ and $\gamma=c/s$ for a constant $c\in[0, 1)$, the ratio between standard errors is lower bounded as
\begin{align}
    \frac{R(\augest)}{R(\stdest)} = \Omega(s^2)
\end{align}
which goes to infinity as $s\rightarrow \infty$. Furthermore, $R(\stdest)\rightarrow 0$ as $s\rightarrow \infty$.
\end{theorem}
\begin{proof}

We first lower bound the standard error of the augmented estimator.
Define $E_1$ as the event that only the lower half of the stairs is sampled, i.e. $\{t:t<s/2\}$, which occurs with probability $(1-\gamma)^n$.
Let $t^\star = \max_i \lfloor t_i \rfloor$ be the largest ``stair'' value seen in the training set.
Note that the min-norm augmented estimator will extrapolate with zero derivative for $t \geq \max_i \lfloor t_i\rfloor$.
This is because on the interval $[t^\star, t^\star+\epsilon]$, the augmented estimator is forced to have zero derivative, and the solution minimizing the second derivative of the prediction continues with zero derivative for all $t \geq t^\star$.
In the event $E_1$, $t^\star \leq s/2-1$, where $t^*=s/2-1$ achieves the lowest error in this event.
As a result, on the points in the second half of the staircase, i.e. $t = \{t\in\sT: t> \frac{s}{2} - 1 \}$, the augmented estimator incurs large error:
\begin{align*}
    R(\augest \mid E_1) &\geq \sum_{t=s/2}^{s} (t-(s/2-1))^2 \cdot \frac{\gamma}{s/2}\\
    &= \sum_{t=1}^{s/2} t^2 \cdot \frac{\gamma}{s/2}= \frac{\gamma}{6}(s^2+2s+1).
\end{align*}
Therefore the standard error of the augmented estimator is bounded by
\begin{align*}
    R(\augest) \geq R(\augest \mid E_1) P(E_1) &=  \frac{\gamma}{6}(s^2+2s+1) (1-\gamma)^n\\
    &\geq \frac{1}{6} \gamma (1-\gamma n)(s^2 + 2s + 1)\\
    &= \Omega(\frac{c-c^2}{s}(s^2 + 2s + 1))= \Omega(s)
\end{align*}
where in the first line, we note that the error on each interval is the same and the probability of each interval is $(1-\delta)\frac{\gamma}{s/2} + \epsilon \frac{\delta}{\epsilon}\cdot \frac{\gamma}{s/2} = \frac{\gamma}{s/2}$.

Next we upper bound the standard error of the standard estimator.
Define $E_2$ to be the event where all points are sampled from $\tline$, which occurs with probability $(1-\delta)^n$.
In this case, the standard estimator is linear and fits the points on $\tline$ with zero error, while incurring error for all points not in $\tline$.
Note that the probability density of sampling a point not in $\tline$ is either $\frac{\delta}{\epsilon}\cdot \frac{1-\gamma}{s/2}$ or $\frac{\delta}{\epsilon}\cdot \frac{\gamma}{s/2}$, which we upper bound as $\frac{\delta}{\epsilon}\cdot \frac{1}{s/2}$.
\begin{align*}
    R(\stdest \mid E_2) = \sum_{t=1}^{s-1} \frac{\delta}{\epsilon}\cdot \frac{1}{s/2} \int_0^\epsilon u^2 du 
    &= \frac{\delta}{\epsilon}\cdot \frac{1}{s/2}O(s \epsilon^3) \\
    &= O(\delta)
\end{align*}
Therefore for event $E_2$, the standard error is bounded as
\begin{align*}
    R(\stdest \mid E_2)P(E_2) &= O(\delta) (1-\delta)^n \\
    &= O(\delta)e^{-\delta n}\\
    &= O(\delta\cdot \frac{s^7-1}{s^7})\\
    &= O(\delta) = O(\frac{\log(s^7)-\log(s^7-1)}{s})=O(1/s)
\end{align*}
since $\log(s^7) - \log(s^7-1) \leq 1$ for $s\geq 2$.
For the complementary event $E_2^c$, note that cubic spline predictors can grow only as $O(t^3)$, with error at most $O(t^6)$. Therefore the standard error for case $E_2^c$ is bounded as
\begin{align*}
        R(\stdest \mid E_2^c) P(E_2^c) &\leq O(t^6)(1-e^{-\delta n})\\
        &=O(t^6)O(\frac{1}{s^7}) = O(1/s)
\end{align*}

Putting the parts together yields 
\begin{align*}
    R(\stdest)=R(\stdest \mid E_2)P(E_2)+ R(\stdest \mid E_2^c)P(E_2^c)\\ \leq O(1/s) + O(1/s) = O(1/s).
\end{align*}
Thus overall, $R(\stdest)= O(1/s)$ and combining the bounds yields the result.
\end{proof}

\section{Robust Self-Training}
\label{sec:app-rst}

We define the linear robust self-training estimator from Equation~\eqref{eqn:linear-x} and expand all the terms.
\begin{align}
    \label{eqn:linear-x-2}
    \thetarst &\in \argmin_{\theta} \Big \{
    \E_{\distribx} [(x^\top \thetainterp - x^\top \theta)^2]
    : \nonumber\\
    & \xstd \theta = \ystd, \max_{x_\text{adv} \in T(x)} (x_\text{adv}^\top\theta - y)^2 = 0 ~\forall x, y \in \xstd, \ystd,  \nonumber \\
    & \E_{\distribx}[\max_{x_\text{adv} \in T(x)} (x_\text{adv}^\top\theta - x^\top\theta)^2] = 0 \Big\}.
\end{align}
Notice that for unlabeled components of the estimator, we assume access to the data distribution $\distribx$ and thus optimize the population quantities.

As we show in the next subsection, we can rewrite the robust self-training estimator into the following reduced form, more directly connecting to the general analysis of adding extra data $\xext$ in min-norm linear regression.
\begin{align}
    \label{eqn:linear-x-3}
    \thetarst &\in \argmin_{\theta} \Big \{
        (\theta - \thetainterp)^\top \sigmapop (\theta - \thetainterp)
    : \xstd \theta = \ystd, \xext \theta = 0 \Big\}
\end{align}
for the appropriate choice of $\xext$, as shown in Section~\ref{sec:app-construct-xext-rst}.
Here, we can interpret $\xext$ as the difference between the perturbed inputs and original inputs. These are perturbations which we want the model to be invariant to, and hence output zero.

\subsection{Robust self-training algorithm in linear regression}
\label{sec:app-construct-xext-rst}

We give an algorithm for constructing $\xext$ which enforces the population robustness constraints.
Suppose we are given $\sigmapop$, the population covariance of $\distribx$.
In robust self-training, we enforce that the model is consistent over perturbations of the labeled data $\xstd$ and (infinite) unlabeled data.
To do this, we add linear constraints of the form $x_{\text{adv}}^\top \theta - x^\top \theta = 0$,
where $x_{\text{adv}}\in T(x)$ for all $x$.
We can view these linear constraints as augmenting the dataset with input-target pairs $(\singleaug, 0)$ where
$\singleaug = x_{\text{adv}} - x$. By assumption, $\singleaug^\top \theta^\star=0$ so these augmentations fit into our data augmentation framework.

However, when we enforce these constraints over the entire population $\distribx$ or when there are an infinite number of transformations in $T(x)$, a naive implementation requires augmenting with infinitely many points.
Noting that the space of augmentations $\singleaug$ satisfying $\singleaug^\top \theta^\star =0$ is a linear subspace, we can instead summarize the augmentations with a basis that spans the transformations.
Let the space of perturbations be $\sT = \cup_{x \in \supp(\distribx), x_{\text{adv}}\in T(x)} x_{\text{adv}} - x$. Note that this space of perturbations also contains perturbations of the original data $\xstd$ if $\xstd$ is in the support of $\distribx$. If $\xstd$ is not in the support of $\distribx$, the behavior of the estimator on these points do not affect standard or robust error.
Assuming that we can efficiently optimize over $\sT$, we construct the basis by an iterative procedure reminiscent of adversarial training.
\begin{enumerate}
    \item Set $t=0$. Initialize $\theta^t=\thetainterp$ and $(\xext)_0$ as an empty matrix.
    \item At iteration $t$, solve for $\singleaug^t = \argmax_{\singleaug \in \sT} (\singleaug^\top \theta^t)^2$. If the objective is unbounded, choose any $\singleaug^t$ such that $\singleaug^\top \theta^t \neq 0$. \label{alg:step-2}
    \item If ${\theta^t}^\top \singleaug^t = 0$, stop and return $(\xext)_t$.
    \item Otherwise, add $\singleaug^t$ as a row in $(\xext)_t$. Increment $t$ and let $\theta^t$ solve \eqref{eqn:linear-x-3} with $\xext=(\xext)_t$.
    \item Return to step~\ref{alg:step-2}.
\end{enumerate}
In each iteration, we search for a perturbation that the current $\theta^t$ is not invariant to. If we can find such a perturbation, we add it to the constraint set in $(\xext)_t$.
We stop when we cannot find such a perturbation, implying that the rows of $(\xext)_t$ and $\xstd$ span $\sT$. The final RST estimator solves \eqref{eqn:linear-x-3} using $\xext$ returned from this procedure.

This procedure terminates within $O(d)$ iterations. To see this, note that $\theta^t$ is orthogonal to all rows of $(\xext)_t$. Any vector in the span of $(\xext)_t$ is orthogonal to $\theta^t$. Thus, if ${\theta^t}^\top\singleaug^t \neq 0$, then $\singleaug^t$ must not be in the span of $(\xext)_t$. At most $d-\text{rank}(\xstd)$ such new directions can be added until $(\xext)_t$ is full rank. When $(\xext)_t$ is full rank, ${\theta^t}^\top \singleaug^t =0$ must hold and the algorithm terminates.

\subsection{Proof of Theorem~\ref{thm:linear-x}}
In this section, we prove Theorem~\ref{thm:linear-x}, which
we reproduce here.
\linearx*
\begin{proof}
    We work with the RST estimator in the form from Equation~\eqref{eqn:linear-x-3}.
    We note that our result applies generally to any extra data $\xext,\yext$.
    We define $\sigmastd = \xstd^\top \xstd$. 
  Let $\{u_i\}$ be an orthonormal basis of the kernel $\Null(\sigmastd +
  \xext^\top \xext)$ and $\{v_i\}$ be an orthonormal basis for
  $\Null(\sigmastd) \setminus \linspan(\{u_i\})$.  Let $U$ and $V$ be the
  linear operators defined by $U w = \sum_i u_i w_i$ and $V w = \sum_i v_i
  w_i$, respectively, noting that $U^\top V = 0$.
  Defining $\pistd \defeq (I - \sigmastd^\dagger \sigmastd)$
  to be the projection onto the null space of $\xstd$, we
  see that there are unique vectors $\rho, \alpha$
  such that
  \begin{subequations}
    \label{eqn:representations-of-three-vecs}
    \begin{equation}
      \label{eqn:burrito-bracket}
      \theta\opt = (I - \pistd) \theta\opt + U \rho + V \alpha.
    \end{equation}
    As $\thetainterp$ interpolates the standard data,
    we also have
    \begin{equation}
      \thetainterp = (I - \pistd) \theta\opt + U w + V z,
    \end{equation}
    as $\xstd U w = \xstd V z = 0$, and finally,
    \begin{equation}
      \label{eqn:represent-X-reg-theta}
      \thetarst = (I - \pistd) \theta\opt + U \rho + V \lambda
    \end{equation}
    where we note the common $\rho$ between
    Eqs.~\eqref{eqn:burrito-bracket} and~\eqref{eqn:represent-X-reg-theta}.
  \end{subequations}

  Using the representations~\eqref{eqn:representations-of-three-vecs}
  we may provide an alternative formulation for the augmented
  estimator~\eqref{eqn:linear-x-2}, using this to prove the theorem.
  Indeed, writing
  $\thetainterp - \thetarst
  = U(w - \rho) + V(z - \lambda)$,
  we immediately have that the estimator has
  the form~\eqref{eqn:represent-X-reg-theta}, with the choice
  \begin{equation*}
    \lambda = \argmin_\lambda
    \left\{ (U(w - \rho) + V(z - \lambda))^\top \sigmapop
    (U(w - \rho) + V(z - \lambda)) \right\}.
  \end{equation*}
  The optimality conditions for this quadratic imply that
  \begin{equation}
    \label{eqn:augmented-opt-conditions}
    V^\top \Sigma V (\lambda - z) = V^\top \sigmapop U(w - \rho).
  \end{equation}
  Now, recall that the standard error of a vector $\theta$
  is $R(\theta) = (\theta - \theta\opt)^\top \sigmapop (\theta - \theta\opt)
  = \norm{\theta - \theta\opt}^2_\sigmapop$, using Mahalanobis norm notation.
  In particular,
  a few quadratic expansions yield
  \begin{align}
    \lefteqn{R(\thetainterp) - R(\thetarst)} \nonumber \\
    & = \norm{U(w - \rho) + V(z - \alpha)}_\sigmapop^2
    - \norm{V(\lambda - \alpha)}_\sigmapop^2 \nonumber \\
    & =
    \norm{U(w - \rho) + V z}_\sigmapop^2
    + \norm{V \alpha}_\sigmapop^2
    - 2 (U(w - \rho) + V z)^\top \sigmapop V \alpha
    - \norm{V \lambda}_\sigmapop^2 - \norm{V \alpha}_\sigmapop^2
    + 2 (V \lambda)^\top \sigmapop V \alpha \nonumber \\
    & \stackrel{(i)}{=}
    \norm{U(w - \rho) + V z}_\sigmapop^2
    - 2 (V \lambda)^\top \sigmapop V \alpha
    - \norm{V\lambda}_\sigmapop^2 + 2 (V\lambda)^\top V \alpha \nonumber \\
    & = \norm{U(w - \rho) + Vz}_\sigmapop^2 - \norm{V \lambda}_\sigmapop^2,
    \label{eqn:simplified-risk-gap-interpolation}
  \end{align}
  where step~$(i)$ used that $(U(w - \rho))^\top \sigmapop V
  = (V(\lambda - z))^\top \sigmapop V$ from the optimality
  conditions~\eqref{eqn:augmented-opt-conditions}.

  Finally, we consider the rightmost term in
  equality~\eqref{eqn:simplified-risk-gap-interpolation}.  Again using the
  optimality conditions~\eqref{eqn:augmented-opt-conditions},
  we have
  \begin{align*}
    \norm{V\lambda}_\sigmapop^2
    & = \lambda^\top V^\top
    \sigmapop^{1/2} \sigmapop^{1/2}
    (U (w - \rho) + V z)
    \le \norm{V\lambda}_\sigmapop
    \norm{U(w - \rho) + V z}_\sigmapop
  \end{align*}
  by Cauchy-Schwarz.
  Revisiting equality~\eqref{eqn:simplified-risk-gap-interpolation},
  we obtain
  \begin{align*}
    R(\thetainterp) - R(\thetarst)
    & = \norm{U(w - \rho) + V z}_\sigmapop^2
    - \frac{\norm{V\lambda}_\sigmapop^4}{
      \norm{V \lambda}_\sigmapop^2} \\
    & \ge \norm{U(w - \rho) + V z}_\sigmapop^2
    - \frac{\norm{V \lambda}_\sigmapop^2 \norm{U(w - \rho) + V z}_\sigmapop^2}{
      \norm{V \lambda}_\sigmapop^2} = 0,
  \end{align*}
  as desired.

  Finally, we show that $\stderr(\thetarst) = \roberr(\thetarst)$.
  Here, choose $\xext$ to contain at most $d$ basis vectors which span $\{x_{\text{adv}}:~x_{\text{adv}}\in T(x),\forall x \in \supp(\distribx)\}$.
  Thus, the robustness constraint $\E_{\distribx}[\max_{x_{\text{adv}}\in T(x)} (x_{\text{adv}}^\top \thetarst - x^\top \thetarst)]=0$ is satisfied by fitting $\xext$.   
  By fitting $\xext$, we thus have
  $x_{\text{adv}}^\top\thetarst - x^\top\thetarst = 0$ for all $x_{\text{adv}}\in T(x), x\in\supp(\distribx)$ up to a measure zero set of $x$.
  Thus, the robust error is
  \begin{align*}
      \roberr(\thetarst) &= \E_{\distribx}[\max_{x_{\text{adv}} \in T(x)} (x_{\text{adv}}^\top \thetarst - x_{\text{adv}}^\top \theta^\star)^2]
                         = \E_{\distribx}[(x^\top \thetarst - x^\top \theta)] = \stderr(\thetarst)
  \end{align*}
  where we used that $x_{\text{adv}}^\top\theta^\star = x^\top \theta^\star$ by assumption. Since $\roberr(\thetarst) \geq \stderr(\thetarst)$, $\thetarst$ has perfect consistency, achieving the lowest possible robust error (matching the standard error).
\end{proof}

\subsection{Different instantiations of the general RST procedure}
\label{sec:app-xreg-robustness}
The general RST estimator (Equation~\ref{eqn:general-x}) is simply a weighted combination of some standard loss and some robust loss on the labeled and unlabeled data.
Throughout, we assume the same notation as that used in the definition of the general estimator. $\xstd, \ystd$ denote the standard training set and we have access to $m$ unlabeled points $\tilde{x}_i, i = 1, \hdots m$.

\subsubsection{Projected Gradient Adversarial Training}
In the first variant, RST + PG-AT, we use multiclass logistic loss (cross-entropy) as the standard loss. The robust loss is the maximum cross-entropy loss between any perturbed input (within the set of tranformations $T(\cdot)$) and the label (pseudo-label in the case of unlabeled data). We set the weights such that the estimator can be written as follows. 
\begin{align}
  \label{eqn:pg-at}
  \hat{\theta}_\text{rst+pg-at}
  \defeq \argmin_{\theta} \bigg\{
  &\frac{1-\lambda}{n} \!\!\!\!\!\!\!\!\! ~~~~~~~~~\sum_{(x, y) \in [\xstd, \ystd]}(1 - \beta) \ell(f_\theta(x), y) + \beta~\ell(f_\theta(\xadv), y) \nonumber \\
  &+ \frac{\lambda}{m} \sum_{i=1}^m
 (1 - \beta) \ell(f_\theta(\tilde{x}_i), f_{\stdest}(\tilde{x}_i)) + \beta~\ell(f_\theta({{}\tilde{x}_{\text{adv}}}_i), f_{\stdest}(\tilde{x}_i))
  \bigg\}, 
\end{align}
In practice, ${\xadv}$ is found by performing a few steps of projected gradient method on $\ell(f_\theta(x), y)$, and similarly ${\tilde{x}}_{{\text{adv}}}$ by performing a few steps of projected gradient method on $\ell(f_\theta(\tilde{x}), f_{\stdest}(\tilde{x}))$. 

\subsubsection{TRADES}
TRADES~\citep{zhang2019theoretically} was proposed as a modification of the projected gradient adversarial training algorithm of~\citep{madry2018towards}. The robust loss is defined slightly differently--it -operates on the normalized logits, which can be thought of as probabilities of different labels. The TRADES loss minimizes the maximum KL divergence between the probability over labels for input $x$ and a perturbaed input $\tilde{x} \in T(x)$. Setting the weights of the different loss of the general RST estimator~\refeqn{general-x} similar to RST+PG-AT above gives the following estimator. 
\begin{align}
  \label{eqn:trades}
  \hat{\theta}_\text{rst+trades}
  \defeq \argmin_{\theta} \bigg\{
  &\frac{(1 - \lambda)}{n} \!\!\!\!\!\!\!\!\!~~~~~~~~~~\sum_{(x, y) \in [\xstd, \ystd]} \ell(f_\theta(x), y) + \beta~KL(p_\theta(\xadv) || p_\theta(x)) \nonumber \\
  &+ \frac{\lambda}{m} \sum_{i=1}^m
\ell(f_\theta(\tilde{x}_i), f_{\stdest}(\tilde{x}_i)) + \beta~KL(p_\theta({{}\tilde{x}_{\text{adv}}}_i) || p_{\stdest}(\tilde{x}_i))
  \bigg\}.
\end{align}
In practice, ${\xadv}$ and ${\tilde{x}}_{{\text{adv}}}$ are obtained by performing a few steps of projected gradient method on the respective KL divergence terms. 

\section{Experimental Details}
\label{app:experiments}

\subsection{Spline simulations}
For spline simulations in Figure~\ref{fig:spline} and Figure~\ref{fig:sample-size}, we implement the optimization of the standard and robust objectives using the basis described in~\citep{friedman2001elements}. The penalty matrix $\minmat$ computes second-order finite differences of the parameters $\theta$. We solve the min-norm objective directly using CVXPY~\citep{diamond2016cvxpy}. 
Each point in Figure~\ref{fig:sample-size}(a) represents the average standard error over 25 trials of randomly sampled training datasets between $22$ and $1000$ samples. Shaded regions represent 1 standard deviation.

\subsection{RST experiments}

We evaluate the performance of RST applied to $\ell_\infty$ adversarial perturbations, adversarial rotations, and random rotations.

\subsubsection{Subsampling \cifar}
\label{app:subsample-cifar}

We augment with $\ell_\infty$ adversarial perturbations of various sizes. In each epoch, we find the augmented examples via Projected Gradient Ascent on the multiclass logistic loss (cross-entropy loss) of the incorrect class. Training the augmented estimator in this setup uses essentially the adversarial training procedure of~\citep{madry2018towards}, with equal weight on both the "clean" and adversarial examples during training.

We compare the standard error of the augmented estimator with an estimator trained using RST. We apply RST to adversarial training algorithms in \cifar~ using 500k unlabeled examples sourced from Tiny Images, as in~\citep{carmon2019unlabeled}.

We use Wide ResNet 40-2 models~\citep{zagoruyko2016wide} while varying the number of samples in \cifar. We sub-sample CIFAR-10 by factors of $\{1, 2, 5, 8, 10, 20, 40\}$ in Figure~\ref{fig:sample-size}(a) and $\{1, 2, 5, 8, 10\}$ in Figure~\ref{fig:sample-size}(b).
We report results averaged from 2 trials for each sub-sample factor.
All models are trained for 200 epochs with respect to the size of the labeled training dataset and all achieve almost 100\% standard and robust training accuracy.

We evaluate the robustness of models to the strong PGD-attack with $40$ steps and $5$ restarts. In Figure~\ref{fig:sample-size}(b), we used a simple heuristic to set the regularization strength on unlabeled data $\lambda$ in Equation~\eqref{eqn:pg-at} to be $\lambda=\min(0.9, p)$ where $p\in[0, 1]$ is the fraction of the original \cifar~dataset sampled. We set $\beta=0.5$.
Intuitively, we give more weight to the unlabeled data when the original dataset is larger, meaning that the standard estimator produces more accurate pseudo-labels.

\begin{figure}[t]
    \centering
    \begin{subfigure}{0.49\textwidth}
      \includegraphics[scale=0.37]{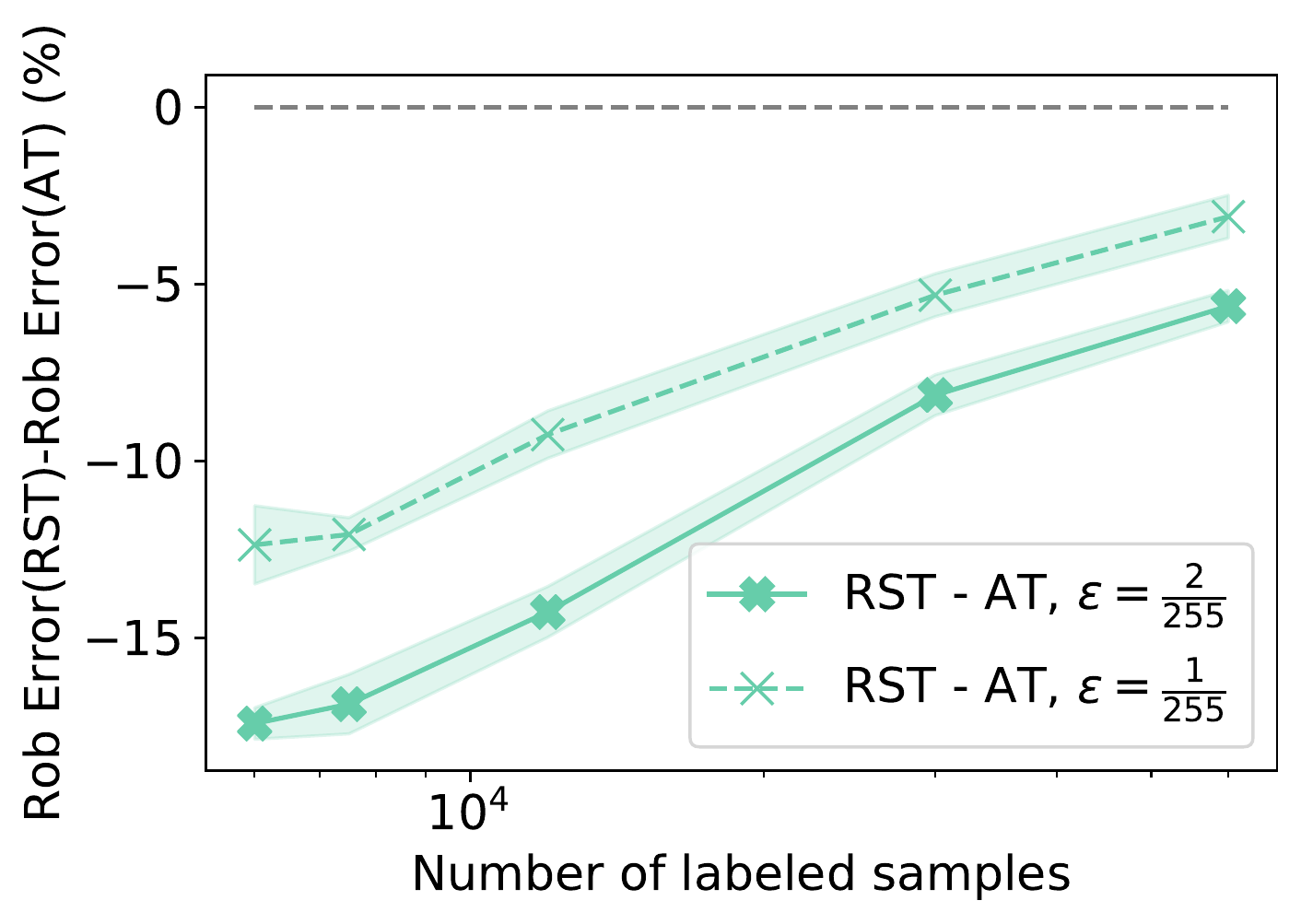}
      \caption{Robust error, \cifar}
    \end{subfigure}
    \begin{subfigure}{0.49\textwidth}
      \includegraphics[scale=0.33]{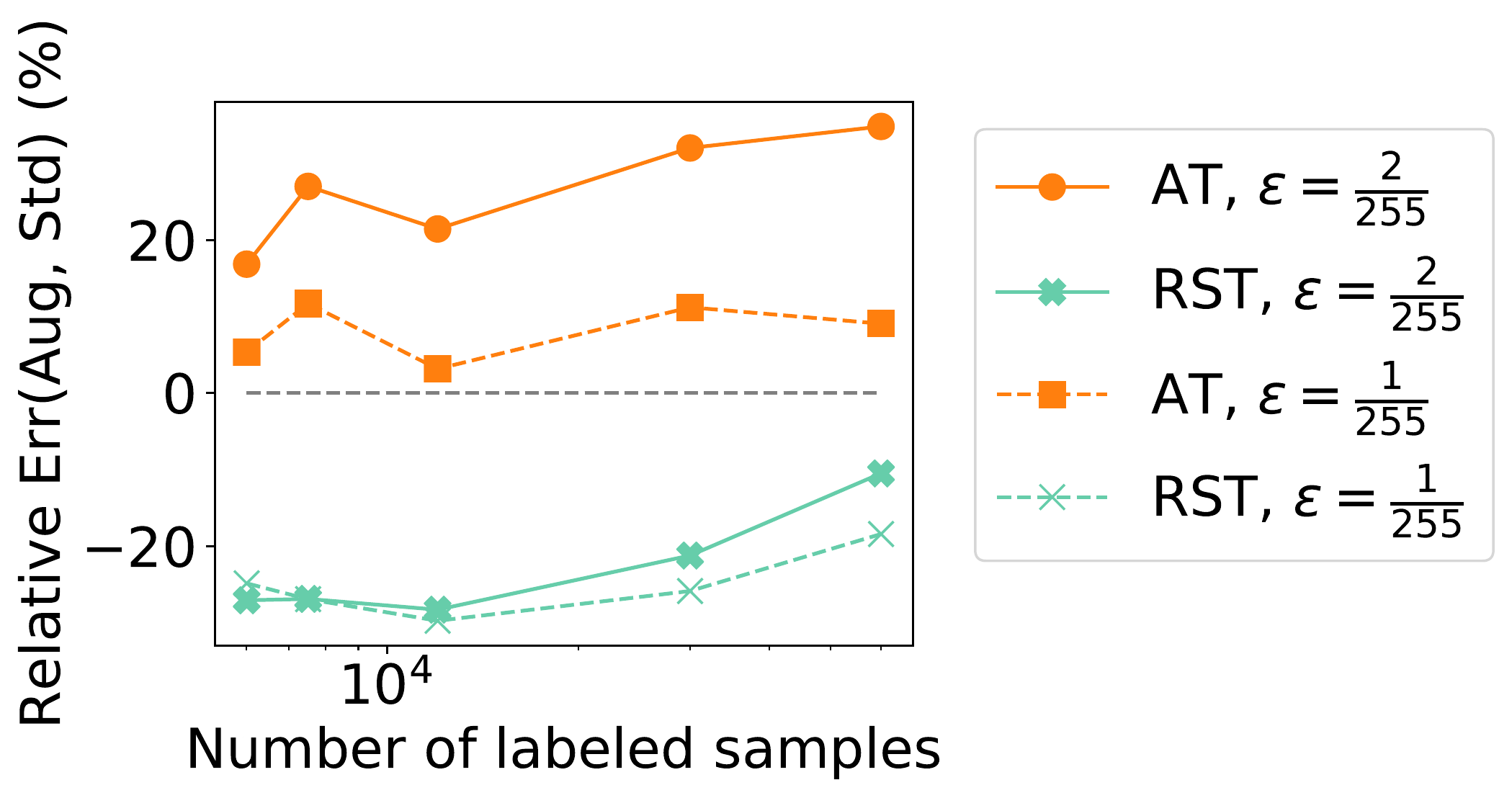}
      \caption{Relative standard error, \cifar}
    \end{subfigure}
    \caption{
      \textbf{(a)}
      Difference in robust error between the RST adversarial training model and the vanilla adversarial training (AT) model for \cifar. RST improves upon the robust error of the AT model by approximately a 15\% percentage point increase for small subsamples and 5\% percentage point increase for larger subsamples of \cifar.
      \textbf{(b)} Relative difference in standard error between augmented estimators (the RST model and the AT model) and the standard estimator on \cifar. We achieve up to 20\% better standard error than the standard model for small subsamples.
    }
    \label{fig:cifar-rob-acc}
\end{figure}

\begin{table*}[t]

\centering
\begin{tabular}{c c c c}
  & Standard & AT  & RST+AT\\
  \midrule
  Standard Acc & 94.63\% & 94.15\% & \textbf{95.58\%}\\
  Robust Acc ($\epsilon=1/255$)  & - & 85.59\% & \textbf{88.74\%}
\end{tabular}

\caption{Test accuracies for the standard, vanilla adversarial training (AT), and AT with RST for $\epsilon=1/255$ on the full \cifar~dataset. Accuracies are averaged over two trials. The robust accuracy of the standard model is near 0\%.}
\label{tab:xreg-numbers}
\end{table*}

\begin{table*}[t]
  \begin{center}
\begin{tabular}{c c c c}
  & Standard & AT & RST+AT\\
  \midrule
  Standard Acc & 94.63\% & 92.69\% & \textbf{95.15\%}\\
  Robust Acc ($\epsilon = 2/255$) & - & 77.87\% & \textbf{83.50\%}
\end{tabular}
\end{center}
\caption{Test accuracies for the standard, vanilla adversarial training (AT), and AT with RST for $\epsilon=2/255$ on the full \cifar~dataset. Accuracies are averaged over two trials. The robust test accuracy of the standard model is near 0\%. \label{tab:xreg-numbers-2}}
\end{table*}

Figure~\ref{fig:cifar-rob-acc} shows that the robust accuracy of the RST model improves about 5-15\% percentage points above the robust model (trained using PGD adversarial training) for all subsamples, including the full dataset (Tables~\ref{tab:xreg-numbers},\ref{tab:xreg-numbers-2}).

We use a smaller model due to computational constraints enforced by adversarial training. Since the model is small, we could only fit adversarially augmented examples with small $\epsilon = 2/255$, while existing baselines use $\epsilon=8/255$. Note that even for $\epsilon = 2/255$, adversarial data augmentation leads to an increase in standard error. We show that RST can fix this. While ensuring models are robust is an important goal in itself, in this work, we view adversarial training through the lens of covariate-shifted data augmentation and study how to use augmented data without increasing standard error. We show that RST preserves the other benefits of some kinds of data augmentation like increased robustness to adversarial examples.

\subsubsection{$\ell_\infty$ adversarial perturbations}
\label{sec:app-linf_adv}
In Table~\ref{table:adv-results}, we evaluate RST applied to PGD and TRADES adversarial training.
The models are trained on the full \cifar~ dataset, and models which use unlabeled data (self-training and RST) also use 500k unlabeled examples from Tiny Images.
All models except the Interpolated AT and Neural Architecture Search model use the same base model WideResNet 28-10.
To evaluate robust accuracy, we use a strong PGD-attack with $40$ steps and $5$ restarts against $\ell_\infty$ perturbations of size $8/255$.
For RST models, we set $\beta = 0.5$ in Equation~\eqref{eqn:pg-at} and Equation~\eqref{eqn:trades}, following the heuristic $\lambda=\min(0.9, p)$ with $p=1$ since we use the entire labeled trainign set. 
We train for 200 epochs such that 100\% training standard accuracy is attained.

\subsubsection{Adversarial and random rotation/translations}

In Table~\ref{table:adv-results} (right), we use RST for adversarial and random rotation/translations, denoting these transformations as $x_\text{adv}$ in Equation~\eqref{eqn:pg-at}.
The attack model is a grid of rotations of up to 30 degrees and translations of up to $\sim10\%$ of the image size.
The grid consists of 31 linearly spaced rotations and 5 linearly spaced translations in both dimensions.
The Worst-of-10 model samples 10 uniformly random transformations of each input and augment with the one where the model performs the worst (causes an incorrect prediction, if it exists).
The Random model samples 1 random transformation as the augmented input.
All models (besides cited models) use the WRN-40-2 architecture and are trained for 200 epochs. We use the same hyperparameters $\lambda,\beta$ as in \ref{sec:app-linf_adv} for Equation~\eqref{eqn:pg-at}.

\section{Comparison to standard self-training algorithms}
\label{sec:app-related}
The main objective of RST is to allow to perform robust training without sacrificing standard accuracy. This is done by regularizing an augmented estimator to provide labels close to a standard estimator on the unlabeled data. This is closely related to but different two broad kinds of semi-supervised learning.
\begin{enumerate}
\item Self-training (pseudo-labeling): Classical self-training does not deal with data augmentation or robustness. We view RST as a a generalization of self-training in the context of data augmentations. Here the pseudolabels are generated by a standard non-augmented estimator that is \emph{not} trained on the labeled augmented points. In contrast, standard self-training would just use all labeled data to generate pseudo-labels. However, since some augmentations cause a drop in standard accuracy, and hence this would generate worse pseudo-labels than RST.

\item Robust consistency training: Another popular semi-supervised learning strategy is based on enforcing consistency in a model's predictions across various perturbations of the unlabeled data~\citep{miyato2018virtual, xie2019unsupervised, sajjadi2016regularization, laine2017temporal}). RST is similar in spirit, but has an additional crucial component. We generate pseudo-labels first by performing standard training, and rather than enforcing simply consistency across perturbations, RST enforces that the unlabeled data and perturbations are matched with the pseudo-labels generated.
  
\end{enumerate}

\section{Minimum $\ell_1$-norm problem where data augmentation hurts standard error}
\label{app:l1_problem}

We present a problem where data augmentation increases standard error for minimum $\ell_1$-norm estimators, showing that the phenomenon is not special to minimum Mahalanobis norm estimators.

\subsection{Setup in 3 dimensions}
Define the minimum $\ell_1$-norm estimators
\begin{align}
  \stdest &= \arg \min \limits_{\theta} \Big \{ \| \theta \|_1 : \xstd \theta = \ystd \Big \} \nonumber\\
  \augest &= \arg \min \limits_{\theta} \Big \{ \| \theta \|_1 : \xstd \theta = \ystd, \xext \theta = \yext \Big \}.  \nonumber
\end{align}

We begin with a 3-dimensional construction and then increase the number of dimensions. Let the domain of possible values be $\sX=\{\xone, \xtwo ,\xthree\}$ where
\begin{align*}
    \xone = [1+\delta, 1, 0], ~~~ \xtwo  = [0, 1, 1+\delta],~~~ \xthree = [1+\delta, 0, 1].
\end{align*}
Define the data distribution through the generative process for the random feature vector $\bx$
\begin{align*}
\bx = 
\begin{cases}
    \xone & \text{w.p. } 1-p\\
    \xtwo & \text{w.p. } \epsilon\\
    \xthree & \text{w.p. } p-\epsilon
\end{cases}
\end{align*}
where $0<\delta<1$ and $\epsilon > 0$.
Define the optimal linear predictor $\theta^\star = \mathbf{1}$ to be the all-ones vector, such that in all cases, $\bx^\top \theta^\star = 2+\delta$.
We define the consistent perturbations as
\begin{align*}
    T(x) = 
\begin{cases}
    \{\xone, \xtwo \} & x \in \{\xone,\xtwo \}\\
    \{\xthree\} & o.w.
\end{cases}
\end{align*}
The augmented estimator will add all possible consistent perturbations of the training set as extra data $\xext$. For example, if $\xone$ is in the training set, then the augmented estimator will add $\xtwo$ as extra data since $\xtwo \in T(\xone)$.
The standard error is measured by mean squared error.

We give some intuition for how augmentation can hurt standard error in this 3-dimensional example.
Define $E_1$ to be the event that we draw $n$ samples with value $\xone$. Given $E_1$, the standard and augmented estimators are
\begin{align}
    \stdest = \left[\frac{2+\delta}{1+\delta},  0, 0\right], ~~~ \augest =[0, 2+\delta, 0].
\end{align}
Note that the $\augest$ has slightly higher norm ($\|\augest\|_1 = 2+\delta > \frac{2+\delta}{1+\delta} = \|\stdest\|_1$).
Since $\xthree^\top \augest=0$ in this case, the squared error of $\augest$ wrt to $\xthree$ is $(\xthree^\top \augest - 2+\delta)^2 = (2+\delta)^2$.
The standard estimator fits $\xthree$ perfectly, but has high error on $\xtwo$.
If the probability of $E_1$ occurring is high and the probability of $\xthree$ is higher relative to $\xtwo$, then the $\augest$ will have high standard error relative to $\stdest$.
Here, due to the inductive bias that minimizes the $\ell_1$ norm, certain augmentations can cause large changes in the sparsity pattern of the solution, drastically affecting the error.
Furthermore, the optimal solution $\theta^\star$ is quite large with respect to the $\ell_1$ norm, satisfying the conditions of Proposition~\ref{prop:simple-complex} in spirit and suggesting that the $\ell_1$ inductive bias (promoting sparsity) is mismatched with the problem.

\subsection{Construction for general $d$}

We construct the example by sampling $\bx$ in 3 dimensions and then repeating the vector $d$ times. In particular, the samples are realizations of the random vector $[\bx; \bx; \bx; \dots; \bx]$
which have dimension $3d$ and every block of 3 coordinates have the same values. Under this setup, we can show that there is a family of problems such that the difference between standard errors of the augmented and standard estimators grows to infinity as $d,n\rightarrow\infty$.
\begin{theorem}
Let the setting be defined as above, where the dimension $d$ and number of samples $n$ are such that $n/d\rightarrow \gamma$ approaches a constant.
Let $p=1/d^2$, $\epsilon=1/d^3$, and $\delta$ be a constant.
Then the ratio between standard errors of the augmented and standard estimators grows as
\begin{align}
    \frac{\stderr(\augest)}{\stderr(\stdest)} = \Omega(d)
\end{align}
as $d,n\rightarrow \infty$.
\end{theorem}
\begin{proof}
We define an event where the augmented estimator has high error relative to the standard estimator and bound the ratio between the standard errors of the standard and augmented estimators given this event.
Define $E_1$ as the event that we have $n$ samples where all samples are $[\xone; \xone; \dots; \xone]$.
The standard and augmented estimators are the corresponding repeated versions
\begin{align}
    \stdest = \left[\frac{2+\delta}{1+\delta},  0, 0, \dots , \frac{2+\delta}{1+\delta},  0, 0\right], ~~~
    \augest = [0, 2+\delta, 0, \dots, 0, 2+\delta, 0].
\end{align}
The event $E_1$ occurs with probability $(1-p)^n + (p-\epsilon)^n$.
It is straightforward to verify that the respective standard errors are
\begin{align*}
    \stderr(\stdest \mid E_1) = \epsilon d^2(2+\delta)^2,~~~\stderr(\augest \mid E_1) = (p-\epsilon)d^2(2+\delta)^2
\end{align*}
and that the ratio between standard errors is
\begin{align*}
    \frac{ \stderr(\augest \mid E_1) }{\stderr(\stdest \mid E_1)} = \frac{p-\epsilon}{\epsilon}.
\end{align*}
The ratio between standard errors is bounded by
\begin{align*}
    \frac{ \stderr(\augest)}{\stderr(\stdest)} &= \sum_{E\in \{E_1, E_1^c\}} P(E) \frac{\stderr(\augest \mid E)}{\stderr(\stdest \mid E)}\\
                                              &>  P(E_1) \frac{\stderr(\augest \mid E_1)}{\stderr(\stdest \mid E_1)}\\
                                              &= ((1-p)^n + (p-\epsilon)^n)(\frac{p-\epsilon}{\epsilon})\\
                                              &> (1-p)^n(d-1)\\
                                              &\geq (1-\frac{n}{d^3})(d-1)\\
                                              &= d-\frac{n}{d^2}-1+\frac{n}{d^3} = \Omega(d)
\end{align*}
as $n,d\rightarrow \infty$, where we used Bernoulli's inequality in the second to last step.
\end{proof}

\end{document}